\documentclass[twocolumn]{article}

\usepackage{microtype}
\usepackage{times}
\usepackage{geometry}
\usepackage{amsmath,amssymb,amsfonts}
\usepackage{amsthm}
\usepackage{mathrsfs}
\usepackage[numbers,compress]{natbib}
\usepackage{graphicx}
\usepackage{multirow}
\usepackage{booktabs}
\usepackage{xcolor}
\usepackage{float}

\usepackage{natbib}
\usepackage{hyperref}
\hypersetup{colorlinks=true, linkcolor=black, citecolor=blue, urlcolor=blue}

\usepackage{textcomp}
\usepackage[title]{appendix}

\theoremstyle{plain}
\newtheorem{theorem}{Theorem}

\newtheorem{corollary}[theorem]{Corollary}
\theoremstyle{definition}

\theoremstyle{remark}

\begin{document}

\twocolumn[
\begin{@twocolumnfalse}

\title{\textbf{Separable neural architectures as a primitive for unified predictive and generative intelligence}}

\author{
  Reza T. Batley$^{1}$ \quad
  Apurba Sarker$^{1\dagger}$\quad
  Rajib Mostakim$^{2\dagger}$\quad
  Andrew Klichine$^{1\dagger}$ \quad
  Sourav Saha$^{1,*}$
}

\date{
  $^1$Kevin T. Crofton Department of Aerospace and Ocean Engineering,
  Virginia Polytechnic Institute and State University, Blacksburg, VA 24060, USA\\
  $^2$Department of Mechanical Engineering,
  Bangladesh University of Engineering and Technology, Dhaka, Bangladesh\\
  $^*$Correspondence: \href{mailto:souravsaha@vt.edu}{souravsaha@vt.edu}
}

\maketitle

\begin{abstract}
Intelligent systems across physics, language and perception often exhibit factorisable structure, yet are typically modelled by monolithic neural architectures that do not explicitly exploit this structure. The separable neural architecture (SNA) addresses this by formalising a representational class that unifies additive, quadratic and tensor-decomposed neural models. By constraining interaction order and tensor rank, SNAs impose a structural inductive bias that factorises high-dimensional mappings into low-arity components. Separability need not be a property of the system itself: it often emerges in the coordinates or representations through which the system is expressed. Crucially, this coordinate-aware formulation reveals a structural analogy between chaotic spatiotemporal dynamics and linguistic autoregression. By treating continuous physical states as smooth, separable embeddings, SNAs enable distributional modelling of chaotic systems. This approach mitigates the nonphysical drift characteristics of deterministic operators whilst remaining applicable to discrete sequences. The compositional versatility of this approach is demonstrated across four domains: autonomous waypoint navigation via reinforcement learning, inverse generation of multifunctional microstructures, distributional modelling of turbulent flow and neural language modelling. These results establish the separable neural architecture as a domain-agnostic primitive for predictive and generative intelligence, capable of unifying both deterministic and distributional representations.
\end{abstract}

\vspace{1em}

\vspace{2em}

\end{@twocolumnfalse}
]

\begin{figure*}[!ht]
    \centering
    \includegraphics[width=\linewidth]{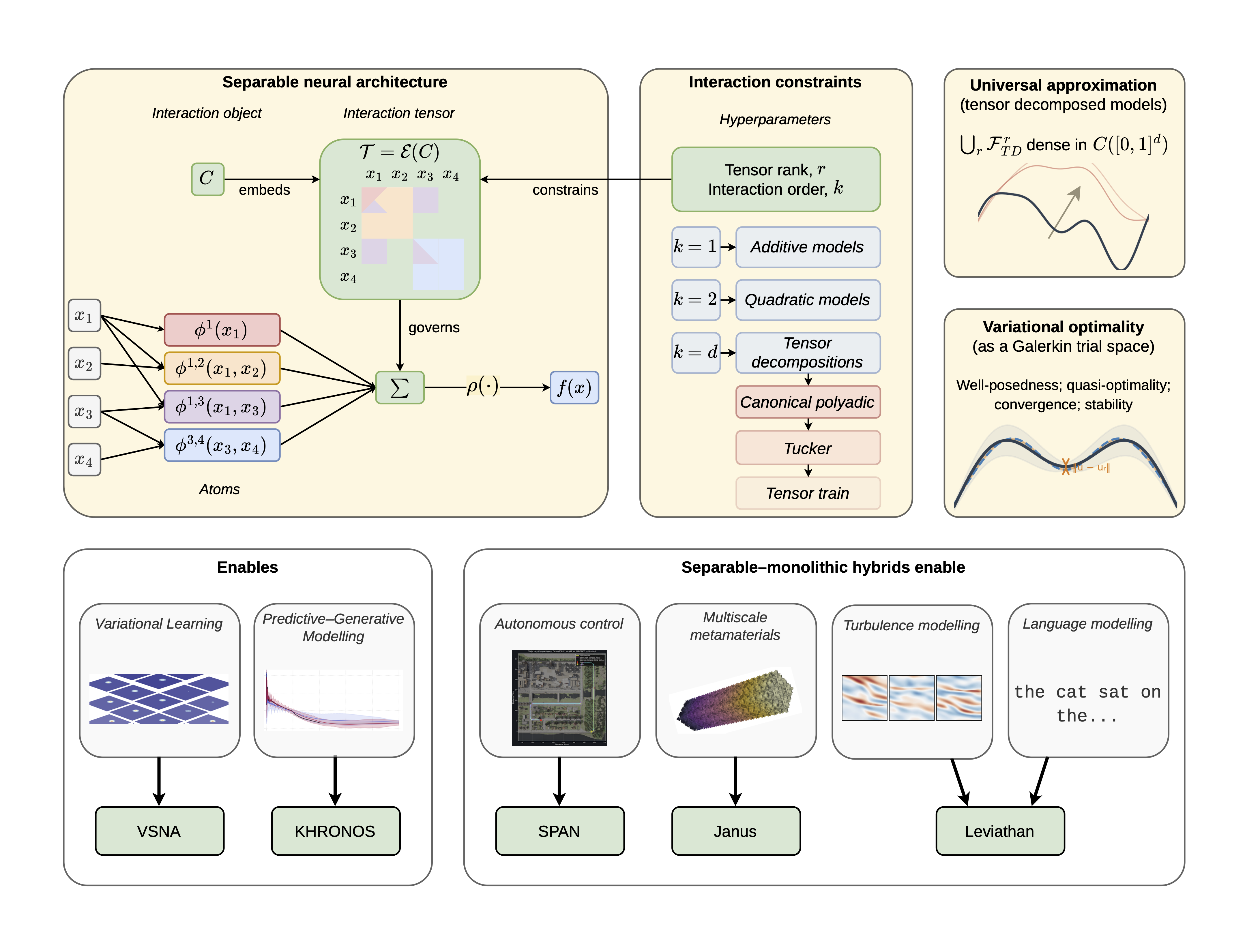}
    \caption{\textbf{The separable neural architecture (SNA) as a unified primitive for predictive and generative intelligence.} \textbf{Top:} The SNA formalises a representational class that constructs high-dimensional mappings by combining lower-arity learnable components (atoms) selected by an interaction tensor. By constraining interaction order and tensor rank, this formalism subsumes generalised additive, quadratic and tensor-decomposed neural models.}
    \label{fig:introduction}
\end{figure*}

Monolithic neural architectures have transformed artificial intelligence. The Transformer, with its ability to model long-range interactions across sequences, has achieved ubiquity in language modelling. Convolutional architectures remain highly effective for local feature extraction. However, systems across physical, linguistic and perceptual domains often exhibit latent factorisable structure that monolithic architectures leave implicit rather than exploit. Moreover, separability is often not a property of a system itself but of the coordinates or representations through which it is expressed. 

Accordingly, this work introduces the separable neural architecture (SNA) as a neural primitive: a rank- and interaction-controlled operator that serves as (i) a standalone model, (ii) a variational trial space, or (iii) a compositional module within larger intelligent systems. Formally, SNAs construct high-dimensional mappings from low-arity learnable components -- \emph{atoms} -- whose interactions are governed by an \emph{interaction object} that can be embedded as a sparse tensor. Expressivity is governed by two controls on the interaction object: its rank $r$ and interaction order $k$, which together control the capacity and sparsity of the learned representation. SNAs thus define a representational class subsuming additive, quadratic and tensor-decomposed neural models within a single formalism.

This primitive permits separable structure to be exploited where it arises, whether explicit, induced by a latent coordinate system or embedded within a larger architecture. When applied at the level of representation, SNAs enable continuous token embeddings that preserve neighbourhood relations in the underlying state space \cite{batley2026leviathan}. Adjacent physical states are thus adjacent in representation, a feature lacking in the discrete lookup embeddings of prevailing neural sequence models. For chaos precludes stable pointwise prediction over extended horizons; modelling must therefore be distributional to avoid nonphysical drift. Under this view, chaotic spatiotemporal dynamics and linguistic autoregression become structurally analogous: both benefit from modelling conditional distributions over sequentially revealed states. 

As a standalone model, the separable neural architecture realises compact predictive--generative intelligence in the form of KHRONOS \cite{batley2025khronos, sarker2026khronos}. KHRONOS instantiates an SNA whose low-rank, separable structure across all dimensions yields a smooth, cheaply invertible interpolant over the input space. Despite containing only hundreds of trainable parameters, it supports accurate prediction and rapid generative inversion to recover entire manifolds of admissible inputs consistent with a queried output. KHRONOS demonstrates that separable primitives can unify prediction and inversion within a single lightweight architecture able to operate in real time on commodity hardware.

This same primitive extends naturally from predictive--generative modelling to variational learning by reinterpreting KHRONOS as a structured Galerkin trial space. In this setting, the SNA is trained directly from a governing operator, yielding a variational separable neural architecture (VSNA) over spatiotemporal-parametric domains. This demonstrates that SNAs may serve as physics-faithful computational representations capable of learning high-dimensional fields from governing operators.

Thus, the present work introduces the SNA as a neural primitive for exploiting latent factorisable structure across intelligent systems (Fig. \ref{fig:introduction}). The SNA is shown to unify predictive and generative modelling across domains. It serves as a standalone architecture enabling parsimonious predictive--invertive learning (KHRONOS); as a trial space for operator-driven learning of high-dimensional spatiotemporal-parametric fields (VSNAs); and as a compositional module within larger intelligent systems, enabling efficient autonomous navigation agents (SPAN), generative inversion of bicontinuous multiscale metamaterials (Janus) and continuous token embeddings for probabilistic sequence modelling (Leviathan).

\section*{Separable neural architectures as a predictive--generative primitive}
\label{CSNA}
\subsection*{Predictive and generative modelling}
\label{PIKHRONOS}

The predictive--generative capacity of separable neural architectures is grounded in a specific subclass, arising when full interaction is permitted ($k=d$) and the interaction tensor is factorised into a rank-$r$ canonical polyadic (CP) decomposition. In this CP-class, atoms factorise into products of univariate sub-atoms $\psi$. Consistent with the convention in generalised decompositions, each atom $\phi$ constitutes a single \emph{mode}: a standalone functional contribution to the global representation. Letting $c^{(j)}$ denote modal weights and $\rho$ an activation function, an element of this class may therefore be written
\begin{align}
    f_r(x;\Theta_{CP})=\rho\left(\sum_{j=1}^rc^{(j)}\prod_{i=1}^d\psi^{(j)}(x_i; \theta_i^{(k)})\right).
\end{align}

The CP-class is grounded in a concrete setting through KHRONOS \cite{batley2025khronos, sarker2026khronos}, a CP-class SNA adopting identity activation  ($\rho(x)=x$), unit modal weights ($c^{(j)}\equiv1$) and B-spline subatoms. This particular CP-class network structure traces its lineage to the interpolating neural network (INN) \cite{park2025unifying} and its Hierarchical Deep Learning Neural Network (HiDeNN) predecessors \cite{saha2021hidenn, zhang2022hidenntd}. KHRONOS has separately demonstrated 100-fold gains over Kolmogorov-Arnold Networks \cite{liu2025kan} on canonical PDE benchmarks \cite{batley2025khronos}. On multi-fidelity aerodynamic field prediction, it achieves accuracy comparable to multilayer perceptrons (MLPs), graph neural networks (GNNs) and physics informed neural networks (PINNs) \cite{raissi2019pinn} with $94-98\%$ fewer parameters \cite{sarker2026khronos}. KHRONOS yields a smooth, separable interpolant over the input coordinate space. For a spline basis of order $P$ over $C_i$ interior cells in each dimension $d$, each subatom takes the form
\begin{align}
    \psi^{(j)}(x_i;\theta^{(j)}_i)=\sum_{c=1}^{C_i+P}\alpha^{(j)}_{i,c}B^{P}_c(x_i).
\end{align}
The extended index $C_i+P$ accounts for domain-exterior ‘‘ghost'' cells required to preserve partition of unity \cite{piegl1997nurbs}.

KHRONOS is demonstrated on a process--structure inverse modelling problem investigated originally in \cite{xie2021mechanistic, fang2022data}, linking thermal histories recorded during directed energy deposition (DED) of Inconel 718 thin-walls to spatially varying mechanical properties of the resultant print. Inverse modelling of 3D printed structures has been a challenge for engineers, as seemingly minor perturbations in the process may result in drastically varying mechanical properties within a single part. A schematic of the experimental setup for this problem is presented in Fig. \ref{fig:SNA}\textbf{a}. The raw thermal signals, recorded by infrared (IR) sensors, are stochastic and nonlinear, as well as long ($10,000$ time indices), whilst available paired data (yield strength, ultimate tensile strength (UTS), and elastic modulus) are sparse (96 samples). The mechanical properties were determined via uniaxial testing of small tensile coupons at different points of a thin wall built with DED.

\begin{figure*}[!ht]
    \centering
    \includegraphics[width=1.0\linewidth]{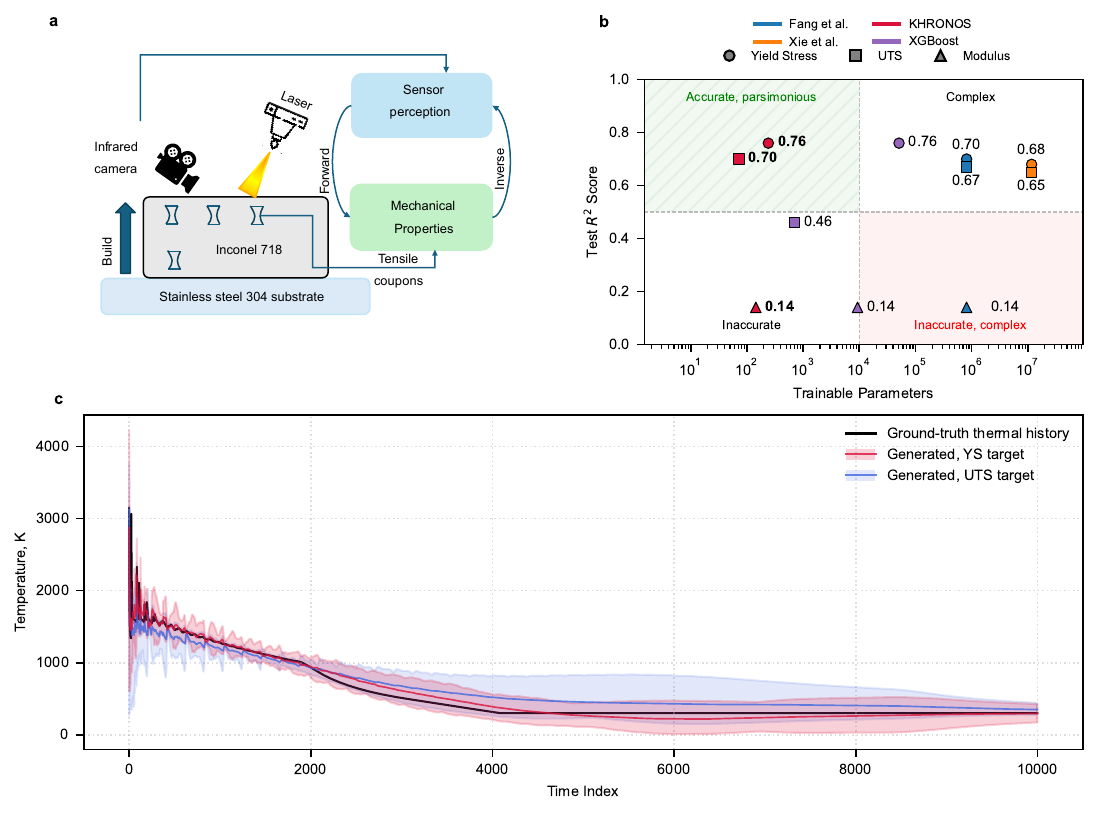}
    \caption{\textbf{Prediction and inversion with a canonical separable neural architecture.} \textbf{a}, A schematic of the experimental setup. A laser directed energy deposition machine builds thin-walled structures layer by layer on a stainless steel 304 substrate, whilst an infrared camera records the evolving thermal field during the build. These measurements are subsequently linked to the mechanical response of the material through tensile testing of extracted coupons. \textbf{b}, Predictive performance versus trainable parameters on the Inconel 718 thermal-history dataset. KHRONOS matches or exceeds prior model accuracy on both yield stress (YS) and ultimate tensile strength (UTS) with up to five orders-of-magnitude fewer parameters \cite{xie2021mechanistic, fang2022data}, and three orders fewer than XGBoost \cite{chen2016xgboost}.
    \textbf{c}, Generative inversion of target mechanical properties to thermal histories. KHRONOS's lightweight structure enables rapid recovery of entire ensembles of plausible histories consistent with queried YS and UTS targets. Here, 47 plausible thermal histories consistent with a target YS of 399.9MPa were recovered in 47.3ms; 64 consistent with a UTS of 670.4MPa in 39.5ms. The illustrated mean and range of converged trajectories closely match the ground-truth thermal history.}
    \label{fig:SNA}
\end{figure*}

Thermal data preprocessing follows \cite{xie2021mechanistic} beginning with a wavelet transform. Previous works fed these high-dimensional representations into monolithic convolutional neural networks (CNNs) -- with approximately 11 million parameters in the modified ResNet18 of \cite{xie2021mechanistic} and 800,000 in the one-dimensional CNN of \cite{fang2022data}. A subsequent principal component analysis (PCA) is introduced, decomposing the preprocessed representations into a low-dimensional latent space. This coordinate transform reveals the factorisability of the thermal physics. Exploiting this, KHRONOS required only 240 parameters for yield stress (YS) and 108 for ultimate tensile strength (UTS) -- a reduction of four to five orders of magnitude. Despite this parsimony, KHRONOS achieved test $R^2$ scores of 0.76 (YS) and 0.70 (UTS), matching or exceeding prior approaches; the only model to achieve the highest score on both metrics. A comparative summary, including an XGBoost baseline \cite{chen2016xgboost}, is visualised in Fig. \ref{fig:SNA}\textbf{b}. All models exhibit poor performance for the material modulus, seemingly saturating at $R^2=0.14$. This is consistent with prior studies: elastic modulus is largely composition-controlled and therefore poorly resolved by thermal histories alone, which are further compounded by IR sensor noise and uncertainties inherent to the deposition process.

Inverting opaque monolithic models requires expensive surrogate optimisation or the training of separate inverse networks. KHRONOS, however, admits easily traced (even analytic) derivatives and is highly lightweight. These properties enable target mechanical properties to be inverted to plausible thermal histories via a structured Newton search. Multiple initialisations produce a low-dimensional manifold of solutions consistent with the queried property. The resulting inversions recover ensembles of thermal histories that closely resemble the ground-truth trajectory with a reasonable uncertainty envelope. This is illustrated in Fig. \ref{fig:SNA}\textbf{c}. The search is so lightweight, in fact, that tens of such histories (47 for YS; 64 for UTS) are generated in under 50ms on commodity CPU hardware. This suggests a route towards an intelligent feedback control system that, deployed in a 3D printing machine, could maintain target mechanical properties throughout the build irrespective of the printing strategy. 

\subsection*{Variational learning of spatiotemporal--parametric fields}
\label{sec:VKHRONOS}

In the sense of classical variational calculus, the variational instantiation of separable neural architectures (VSNAs) for the solution of PDEs follows the same rank-controlled, separable structure that defines predictive SNAs. The key distinction is that VSNAs learn directly from governing operators rather than from data. In this setting, the separable representation acts as a global trial space over an entire spatiotemporal-parametric domain, treating it as a continuous physical manifold to be recovered. 

VSNAs sit at the intersection of established paradigms for physics-based solution of spatiotemporal-parametric fields. Comparable finite-element approaches discretise the entire spatiotemporal--parameter domain with high-dimensional shape functions, but encounter the ‘‘curse of dimensionality'': an exponential growth in degrees of freedom. Physics-informed neural networks (PINNs; \cite{raissi2019pinn}) instead parameterise the solution space with a monolithic neural field trained by minimising strong-form PDE residuals together with boundary and initial condition losses, imposed only softly, and also lack variational optimality guarantees. Proper generalised decomposition (PGD; \cite{Chinesta2011}) employs similar low-rank tensor products to address high-dimensional PDEs but relies on a ‘‘greedy'' training strategy: modes are optimised sequentially and then frozen. This prevents communication between rank components, often requiring a higher rank to reach a given accuracy than global training would. VSNAs unify these perspectives by combining operator-driven variational training with a separable and learned representation -- structurally akin to PGD -- trained globally, as in neural approaches.

\begin{figure*}[!ht]
    \centering
    \label{fig:VSNA}
    \includegraphics[width=0.97\linewidth]{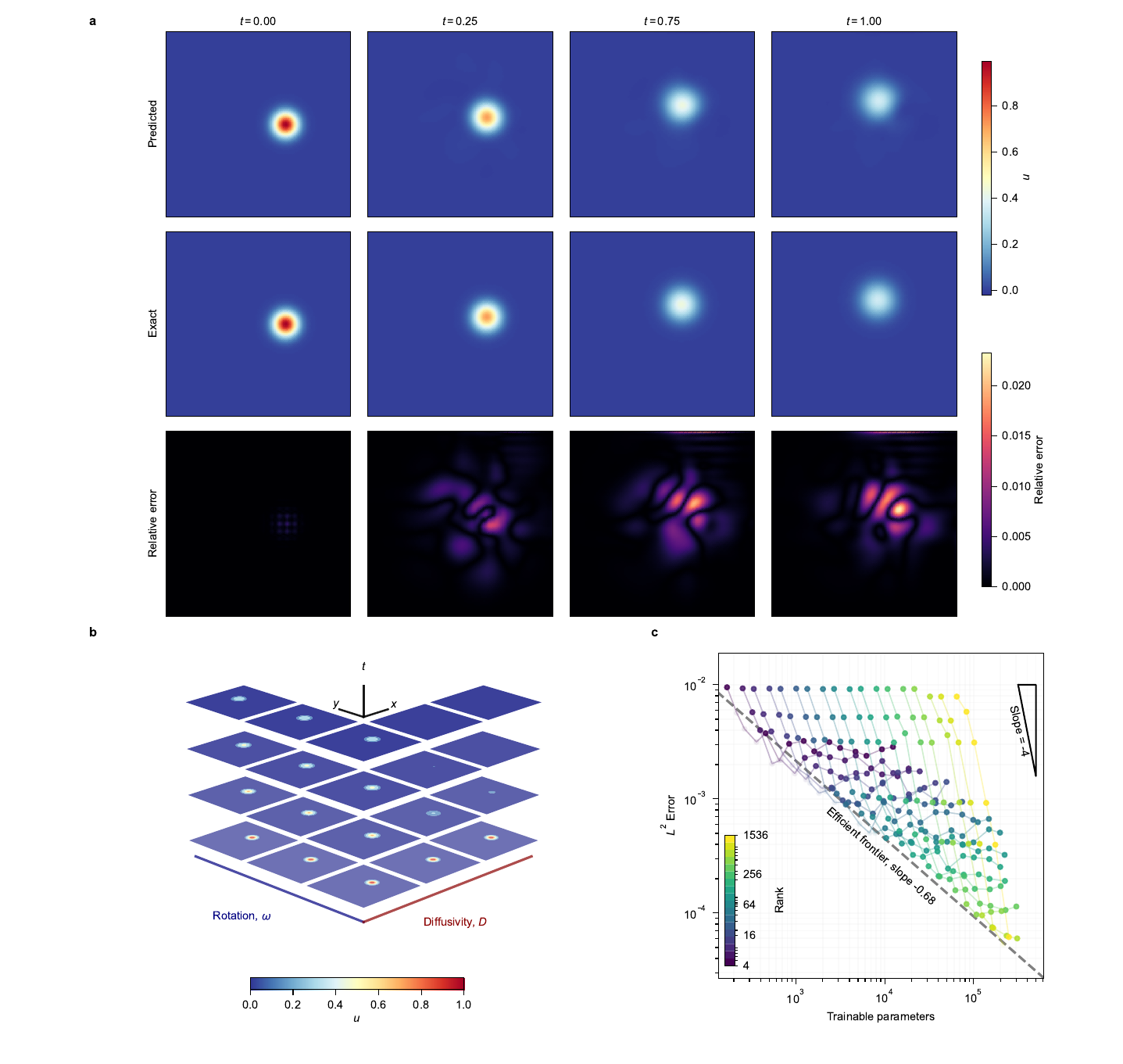}
    \caption{\textbf{Variational separable neural architectures recover high-dimensional PDE solution manifolds with favourable scaling.} \textbf{a}, Spatiotemporal evolution of the field for fixed $\omega=\frac \pi 3$ and $D=0.001$. The top and middle rows compare KHRONOS's predicted solution with the exact, and the bottom shows KHRONOS's relative error. \textbf{b}, The six-dimensional spatiotemporal-parametric advection-diffusion field learned by KHRONOS. Stacked ($x-y$) spatial slices across time $t$ are shown across rotation-diffusivity ($\omega-D$) parameter space, illustrating the recovery of the entire solution manifold in a single global representation. \textbf{c}, Approximation ($L^2$) error versus trainable parameters for the same system under refinement of rank $R$ and resolution $C$. Rank-isolines are connected and colour-coded. Along rank-isolines, errors decrease with resolution at slope $=-4$ before saturating at the rank capacity limit. Across ranks, an efficient frontier emerges (fitted slope $\approx-0.68$ in log-log space), sustained across four orders of magnitude in parameter count.}
\end{figure*}

The VSNA instance examined herein is KHRONOS, the CP-class SNA with each coordinate direction represented by a learned B-spline basis expansion. Although identical in functional form to the predictive model, it is here interpreted variationally as forming a finite-rank trial space over the spatiotemporal-parametric domain. Such separable trial spaces are dense in the underlying Hilbert space under mild assumptions. It follows immediately that the global solution is approximable to arbitrary precision as rank $R$ and resolution $C$ are increased. Under standard assumptions on the governing operator -- most notably boundedness and coercivity of a bilinear form -- KHRONOS converges in the same limits.

The PDE solution is obtained by physics-based training via least-squares minimisation of the governing operator residual. To illustrate this principle, a six-dimensional spatiotemporal--parametric advection--diffusion system is considered:
\begin{align}
    \frac{\partial u}{\partial t}+\boldsymbol{U}\cdot\nabla u-D\nabla^2u=0.
\end{align}
The field $u$ evolves over spatial coordinates $(x,y,z)\in[0,1]^3$ with homogeneous Dirichlet boundary conditions, as well as time $t\in[0,1]$, angular velocity $\omega\in[0,\tfrac \pi 3]$ and diffusivity $D\in[0.001, 0.01]$. The motion of an initial Gaussian plume is driven by a two-dimensional solid-body rotating wind $\boldsymbol{U}=[-\omega(y-\tfrac1 2),\omega(x-\tfrac1 2),0]^T$. In physical application, such a system might model the
transport and dissipation of a scalar quantity -- energy, aerosols, perhaps pollutants -- within a rotating fluid domain \cite{vallis2017atmospheric}.

KHRONOS captures the full six-dimensional solution manifold as a low-rank separable field over space, time and parameters. This representation permits the continuous field to be queried at arbitrary locations in space, time \emph{and} parameter space. Whereas a classical FEM or standard PINN approach would require a full re-solve for each desired parameter combination, KHRONOS provides the full space-time field, queried in milliseconds. Fig. \ref{fig:VSNA}\textbf{a} illustrates representative two-dimensional spatial slices at $t=0,0.5$ and $1$ of the learned six-dimensional manifold at $\omega=\frac \pi 4, D=0.001$. KHRONOS's prediction is compared against a semi-analytic proxy of the governing system, and space-wise absolute errors are shown. KHRONOS reproduces both rotational transport and diffusive spread -- albeit mild at this $D$ -- with high fidelity across time. Errors remain smooth and spatially structured. Having established that the CP-class VSNA recovers the coupled spatiotemporal--parametric dynamics, the natural question is how this accuracy scales with computational resources.

Fig. \ref{fig:VSNA}\textbf{b} quantifies approximation accuracy under joint refinement of rank $R$ and resolution $C$. Error initially decreases systematically with slope $=-4$ with $C$-refinement -- as is expected with cubic B-splines -- but saturates once rank capacity is reached. This combined effect produces an efficient frontier sustained across four orders of magnitude in trainable parameters $N$. This frontier follows an empirical scaling $\|e\|_{L^2}\approx0.24N^{-0.68}$, consistent with the theoretical convergence rate of $-\frac p d =-\frac4 6$ for cubic B-splines in six dimensions. 

The improved intercept compresses the parameters needed to achieve a target error by three orders of magnitude compared to six-dimensional cubic B-spline FEM. This understates the advantage: achieving comparable error would require a mesh density exceeding memory limits by many orders of magnitude, with a corresponding $O(N^{18})$ explosion in solver complexity for a direct solve.  

Collectively, these results establish the separable neural architecture as a highly capable standalone primitive. Whether learning from sparse, noisy data to predict and invert thermal histories, or acting as a Galerkin trial space for the solution of high-dimensional PDEs, the SNA exploits separable structure where it may exist. Having demonstrated its efficacy as an isolated primitive -- further validated on high-dimensional regression and a nonlinear PDE (see Appendix) -- the natural progression is as a structural inductive bias within larger, separable--monolithic composite learning systems.

\section*{Composite learning systems}

\subsection*{Generative inversion of multiscale metamaterials}

\begin{table*}[!h]
\centering
\caption{\textbf{L-BOM dataset features}. Inputs are symmetry-reduced unit cells; outputs comprise the 21-component elastic tensor, volume fraction and permeability.}\label{tab:dataset}
\begin{tabular}{l c l}
\toprule
\textbf{Feature} & \textbf{Size} & \textbf{Description} \\ 
\midrule
\multicolumn{3}{l}{\textbf{Input}} \\
Voxel Grid & $64^3$ & Origin-anchored octant of L-BOM. Binary field. \\ 
\midrule
\multicolumn{3}{l}{\textbf{Outputs}} \\
Normal Stiffness & 3 & Axial ($C_{1111}, C_{2222}, C_{3333}$). \\
Normal Coupling & 3 & Off-diagonal ($C_{1122}, C_{1133}, C_{2233}$). \\
Shear Stiffness & 3 & Diagonal ($C_{1212}, C_{1313}, C_{2323}$). \\
Shear Cross & 12 & Identically zero due to symmetry. \\
Volume & 1 & Solid volume fraction. \\
Permeability & 1 & Fluid transport capacity. \\ 
\bottomrule
\end{tabular}
\end{table*}

Janus \cite{batley2026janus} is a bidirectional framework for generative inversion of three-dimensional multiscale metamaterials, in which the SNA serves as a compositional module within a larger intelligent system. At the macroscale, topological optimisation typically demands continuously varying and specific mechanical properties to maximise structural efficiency \cite{bendsoe1988generating, bendsoe1999material, fleck2010micro}. However, designing these property fields at the microscale requires solving an ill-posed inverse homogenisation problem. Traditional concurrent multiscale approaches are computationally prohibitive \cite{rodrigues2002hierarchical, feyel2000fe2}. Recent data-driven approaches either rely on pre-computed libraries \cite{panetta2015elastic, wang2026lbom} or sample from monolithic generative models \cite{wang2020deep, bastek2022inverting}, typically suffering from limited property coverage and disjointed boundaries \cite{garner2019compatibility, zheng2021datadriven}. Janus circumvents these limitations by treating the continuous physical state as a separable embedding. Each unit cell microstructure is generated via gradient-based maximum a posteriori (MAP) inversion \cite{bora2017compressed, yeh2017semantic} in a highly compressed latent space. This approach encourages topological veracity (‘‘on manifold'' behaviour) and perfect boundary connectivity.

\begin{figure*}[!ht]
    \centering
    \includegraphics[width=0.9\linewidth]{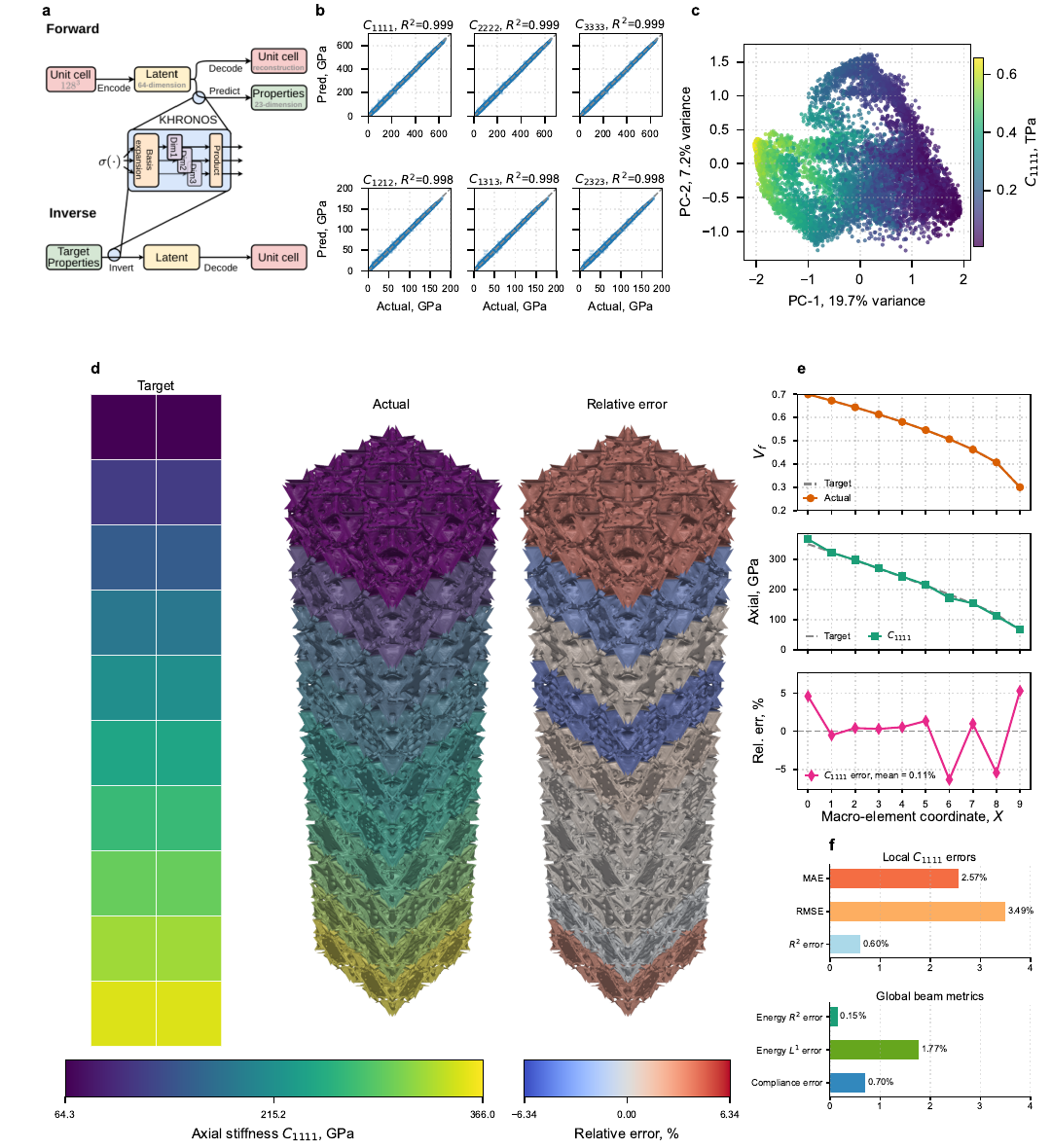}
    \caption{\textbf{Bidirectional generative framework and realisation of seamless, multiscale metamaterials.} \textbf{a}, Schematic of Janus's architecture. A three-dimensional convolutional autoencoder encodes a unit cell voxelised microstructure into a 64-dimensional latent space from which it learns to reconstruct them. A separable neural architecture head, similar to KHRONOS, predicts physical properties from the latent. This head is readily inverted to generate a new microstructure given target properties. \textbf{b}, Forward prediction accuracy of Janus on key components of the stress tensor from a held-out test set, demonstrating near-perfect correlation. \textbf{c}, Principal component analysis (PCA) of the latent space coloured by axial stiffness $C_{1111}$, highlighting the smooth manifold learned by the network. \textbf{d}, Macroscale $C_{1111}$ stiffness targets as prescribed by the cantilever beam model, volumetric rendering of the 40-cell multiscale beam with Janus-designed microstructures, and rendering shaded with local relative errors as determined by FFT homogenisation. \textbf{e}, Beamwise validation of the designed property field. Actual volume fraction exactly tracks the target, and axial stiffness closely agrees across the beam with low relative errors. \textbf{f}, Summary of local stiffness-field and global beam-level metrics. Local $C_{1111}$ errors remain below $3.5\%$, whilst global metrics remain below $2\%$, confirming the intended structural-level response.}
    \label{fig:janus}
\end{figure*}

Janus is validated on a macro-scale beam comprising $10\times2\times2$ unit cells. The target property field requires a monotonic reduction in solid volume fraction $V_f$ from 0.65 at the root to 0.25 at the tip, paired with a corresponding gradient in the primary load-bearing axial stiffness $C_{1111}$ of 350GPa at the root down to 50GPa at the tip. Specifically, this gradient is derived from a cantilever beam model: the bending moment distribution under tip loading prescribes a monotonically decreasing stiffness field, with local Young's modulus scaled from the volume fraction via a SIMP power law \cite{bendsoe1999material}. The first step involves training Janus on a large-range, boundary-identical, bicontinuous and open-cell microstructure (L-BOM) dataset \cite{wang2026lbom}. This dataset contains 10,770 boundary-masked $128\times128\times128$ microstructures. Due to cubic symmetry, origin-anchored $64\times64\times64$ octants are used as input data. 

Within Janus, the separable head learns to predict the 23 physical properties, as detailed in Table \ref{tab:dataset}, from 64-dimensional latent codes generated by the encoder from these octants. A separate decoder head learns to reconstruct the original octant from the same latent code. This schematic is visualised in Fig. \ref{fig:janus}\textbf{a}. In this phase, Janus achieves a reconstructive binary cross-entropy loss of 8\%, an $R^2=0.82$ for permeability and $R^2>0.99$ for all normal stiffness and coupling terms of the stress tensor from the latent space. Parity plots of the axial and shear stiffnesses are shown in Fig. \ref{fig:janus}\textbf{b}. Unlike the isotropic clouds typical of probabilistic generative models, Janus's latent space is structured and smooth, as can be seen in Fig. \ref{fig:janus}\textbf{c}. Janus achieves a cycle consistency of $2\%$, indicating that the latent space is stable under encode--decode cycles.

Janus is subsequently deployed for generative inversion. Maximum a posteriori inversion guards against off-manifold latent codes and aids in avoiding gradient hallucination -- the discovery of pathological latent codes that ‘‘trick'' the predictor whilst diverging from true physics (cf. \cite{goodfellow2015explaining}). Ensembling is also used: for each unit cell, Janus produces 16 candidate microstructures in parallel with the lowest error -- as determined by a final FFT solve -- selected for the macro-structure. Since Janus learns a continuous topological field, volume-preserving thresholding is used for binarisation to ensure exact adherence to desired volume fraction. This entire process takes two-and-a-half minutes to construct the multiscale beam composed of 84 million voxels. 

The stiffness targets prescribed by the cantilever beam model are visualised in Fig. \ref{fig:janus}\textbf{d}, alongside the Janus-generated beam rendered by FFT-validated stiffness and by local relative error. Fig. \ref{fig:janus}\textbf{e} confirms that actual volume fraction exactly tracks the target, and that axial stiffness closely agrees across the beam with a mean relative error (signed) of $0.1\%$ for the primary design objective $C_{1111}$. Fig. \ref{fig:janus}\textbf{f} confirms these modest local errors, with mean absolute error (MAE) of $2.57\%$, root mean squared error (RMSE) of $3.49\%$ and $R^2$ score of 0.994 for $C_{1111}$. Global energy distribution metrics show close agreement -- with a correlation of $0.999$ and $L^1$ error of 1.77\%. Crucially, tip deflection of the generated multiscale beam -- the macroscale quantity prescribing local stiffness objectives -- agrees to within 0.7\%.

\subsection*{Distributional sequence modelling of turbulence}

Many systems of interest are high-dimensional, stochastic and inherently distributional -- the objective being the characterisation of admissible futures, not just pointwise prediction. Leviathan \cite{batley2026leviathan} is a composite learning system that extends the SNA formalism to this regime, applying it to the distributional prediction of turbulence: a stringent test in which even short-horizon forecasts must represent ensembles of feasible future states.

Leviathan is evaluated on two-dimensional incompressible turbulence from the PDEBench suite \cite{takamoto2022pdebench, takamoto2022dataset}, simulated at Mach 0.1 with viscosity and dissipation parameters $\eta=10^{-8},\zeta=10^{-8}$ and periodic boundary conditions. The resulting fields are resolved on a $512\times512$ grid over 21 time steps. From each field, 64 non-overlapping $64\times64$ patches are extracted and treated as independent spatial streams, greatly amplifying the training corpus. The problem becomes one of learning local, translationally-invariant behaviour in open-boundary turbulent flow.

The predominant neural surrogates in this setting are deterministic operator learners. The Fourier Neural Operator \cite{li2021fno} learns mappings from $u(t)$ to $u(t+1)$ via pointwise regression. This approach is mirrored by DeepONet \cite{lu2021deeponet}, \emph{Separable} DeepONet \cite{mandl2025sepdeeponet} and Galerkin-based Transformer architectures \cite{Cao2021transformer}. Such pointwise-deterministic approaches achieve strong short-horizon accuracy and approximate the local evolution operator effectively.

In chaotic systems, however, neighbouring trajectories diverge exponentially under autoregressive rollout. Although governed by deterministic equations, chaotic evolution is effectively probabilistic for prediction as infinitesimal state uncertainty -- such as the floating-point noise floor -- inevitably progresses into macroscopic variability over time. Treating evolution as a deterministic mapping therefore imposes a misaligned inductive bias. In pointwise-regressive operator learning this renders long-horizon trajectories nonphysical; they effectively ‘‘fall off'' the attractor and fail to preserve inertial-range statistics and physical properties \cite{jiang2023training}. One such off-attractor failure mode manifests as mean-state drift, yielding biased climatological averages under autoregressive rollout in modelling weather systems \cite{scher2019weather}.

Leviathan instead learns a conditional distribution over admissible future states, an inductive bias better suited to the finite-precision reality of chaotic turbulence. Uncertainty becomes the primary modelling objective; Leviathan learns an ensemble of feasible next states conditioned on the prior. Herein lies the structural analogy: Leviathan treats chaotic spatiotemporal evolution no differently from linguistic autoregression, learning turbulence as a language in continuous embedding space. The manifold learned by Leviathan's embeddings thus represents emergent factorisability in the underlying dynamics. In exploiting this separability, Leviathan inaugurates a foundation-model paradigm for turbulence.

\begin{figure*}[!ht]
    \centering
    \includegraphics[width=\linewidth]{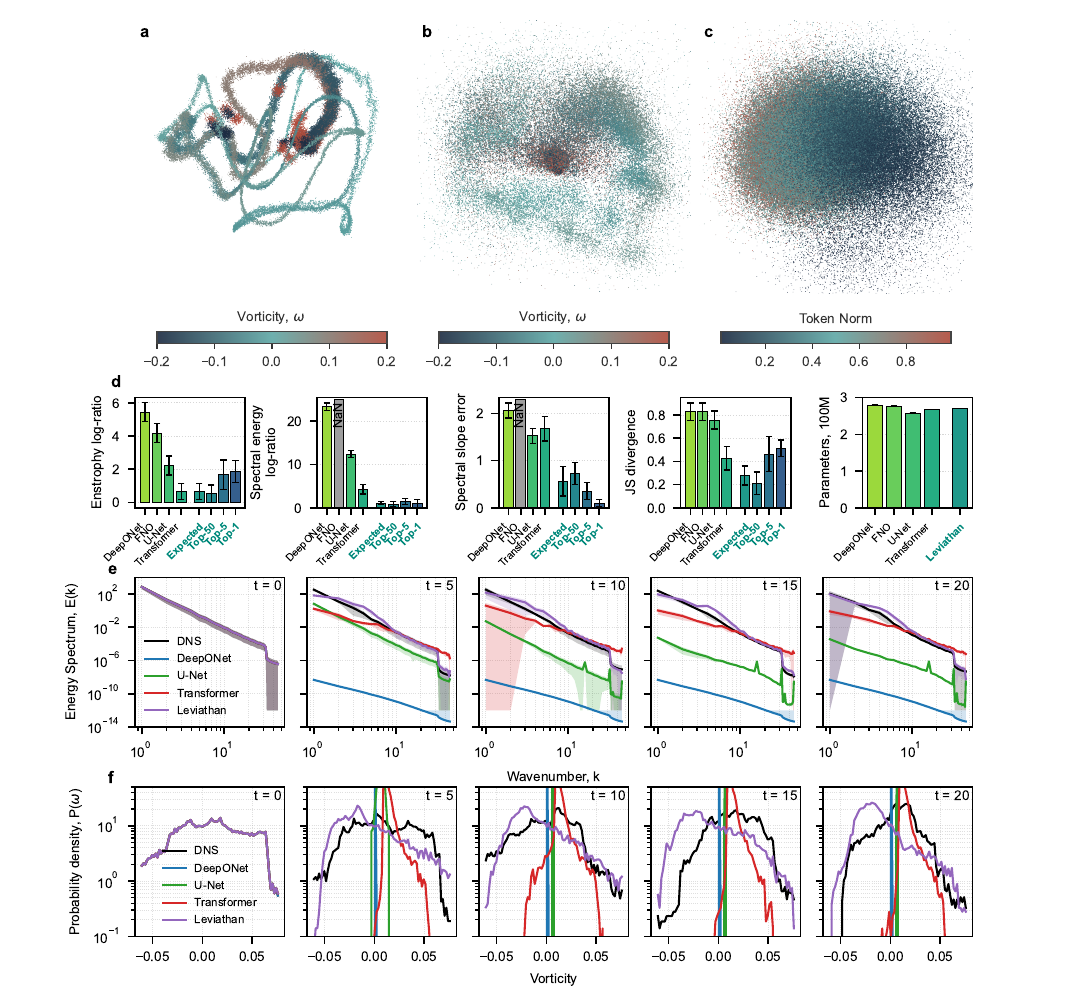}
    \caption{\textbf{Analysis of Leviathan as a foundation model for turbulence across three rollout seeds.} \textbf{a-c}, three-dimensional principal components of the embeddings of the entire vocabulary set. \textbf{a}, Leviathan generates a continuous embedding manifold of low intrinsic dimensionality, with the visualised components explaining $85\%$ of the variance. \textbf{b}, A dense Transformer embeds isotropically, explaining only $14\%$ of the variance. \textbf{c}, the isotropic cloud of Leviathan's embedding space when trained on the unstructured \texttt{o200k\_base} tokeniser. Despite the mathematical structure of quantised vorticity, the dense embedding space in \textbf{b} more closely resembles that of an unstructured language tokeniser. \textbf{d}, Quantitative validation of long-horizon -- 20 timestep -- physical consistency. Leviathan, under four sampling techniques (expectation, top-50, top-5, greedy) outperforms deterministic operators (DeepONet, Fourier neural operator, U-Net) across all metrics (left to right: enstrophy log-ratio error, total spectral energy log-ratio error, spectral slope error, Jensen--Shannon divergence) when controlling for parameters. The dense Transformer is  competitive in enstrophy and Jensen-Shannon divergence. \textbf{e}, Evolution of radial energy spectra in time, with Leviathan best maintaining inertial-range statistics. The deterministic operators rapidly fall away from the direct numerical simulation (DNS) ground truth. The Fourier neural operator fades to a constant field in a single step, with flat spectrum. \textbf{f}, Evolution of the probability density function $P(\omega)$ of vorticity. Deterministic models drift catastrophically to a non-physical mean state -- a delta distribution -- whereas Leviathan preserves the heavy-tailed structure of the chaotic attractor. The dense Transformer retains some structure, avoiding collapse to a mean state.}
    \label{fig:leviathan}
\end{figure*}

\begin{figure*}[!ht]
    \centering 
    \includegraphics[width=\linewidth]{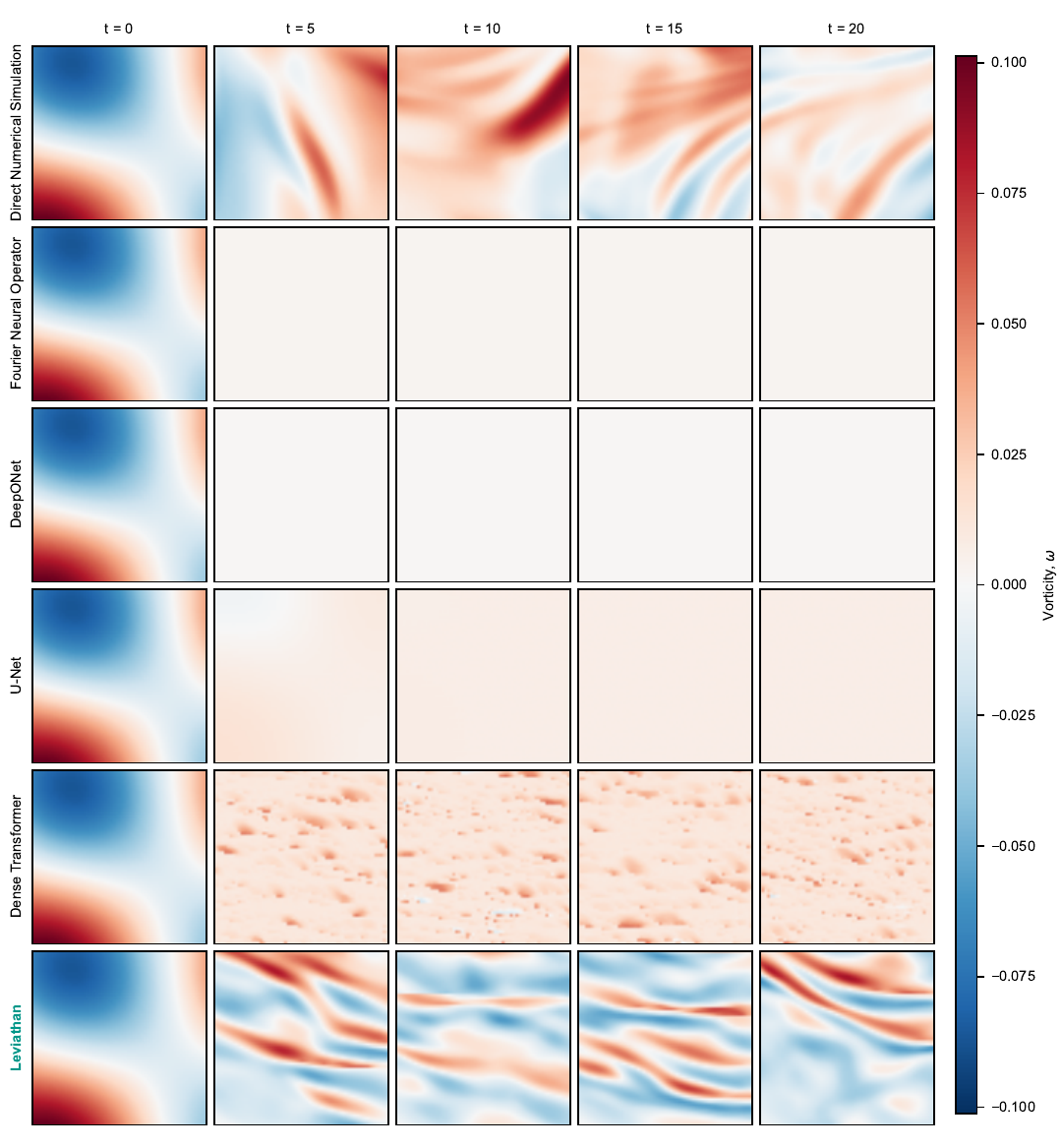}
    \caption{\textbf{Autoregressive rollout of turbulent flow over long horizons.} A comparative visualisation of two-dimensional incompressible turbulence, with ground truth generated via direct numerical simulation at Mach 0.1 and a Reynolds number of 10 million. The ground truth is compared against state-of-the-art deterministic operators (Fourier neural operator, DeepONet, U-Net), as well as a dense Transformer and Leviathan both sampled via expectation. The operator learners decay to a mean state; the dense Transformer preserves a degree of structure but is fundamentally handicapped by its embedding approach. Leviathan qualitatively tracks the ground truth.}
    \label{fig:rollout}
\end{figure*}

Leviathan achieves this through its generator module \cite{batley2026leviathan}, a neural token-embedding engine that base-decomposes tokens into coordinates  and maps them into a seeding space. In the following exposition, input vorticities are quantised to \texttt{uint16} precision and base-256-decomposed into a two-dimensional coordinate system then embedded into a 128-dimensional seeding space. The manifold formed by these intermediate embeddings is learned by a separable neural architecture and supplied to the Transformer backbone. This construction preserves neighbourhood relations: adjacent physical states remain adjacent in representation, unlike conventional static embeddings where neighbouring states occupy entirely unrelated slots.

Each Leviathan attention block uses a causality-respecting Prefix-LM mask \cite{raffel2023exploring, wang2022prefixlm} adapted to spatiotemporal fields. The prior state $p(t)$ is processed bidirectionally, seeing full spatial context, whereas the next state $p(t+1)$ is generated autoregressively: each token $p(t+1,i)$ attends to the full history $p(t)\cup p(t+1,<i)$ but remains masked from future tokens $p(t+1,\geq i)$. Crucially, the mask ensures that $p(t)$ does not attend to $p(t+1)$, preventing acausal information leakage across time. 

Training proceeds via maximisation of the conditional likelihood of the next state given the prior. Upon deployment, Leviathan is evaluated under free autoregression. Predictions are recursively sampled from the model distribution, exposing the long-horizon stability in chaotic evolution. The central test in this regime is whether generated trajectories remain physical -- on-attractor.  

Fig. \ref{fig:leviathan}\textbf{a} illustrates the embedding space of turbulent states by means of three-dimensional principal component analysis (PCA). Rather than the typical isotropic clouds -- typified in the embedding manifolds of both the dense Transformer (Fig. \ref{fig:leviathan}\textbf{b}) and the language-centric \texttt{o200k\_base} tokeniser (Fig. \ref{fig:leviathan}\textbf{c}) -- Leviathan offers instead a topology that is elongated, smooth and of low intrinsic dimensionality. These three components capture $85\%$ of the explained variance, compared with only $14\%$ for the dense transformer -- a quantitative confirmation of the structural distinction. By preserving adjacency -- from physical space to embedding space -- the separable primitive provides the model an inductive bias lacking in standard sequence models.

This structural advantage translates directly into long-horizon physical consistency. Under 20-step autoregressive rollouts, the tested state-of-the-art deterministic operators (DeepONet, Fourier neural operator (FNO) and U-Net) suffer catastrophic drift to nonphysical mean states, as evidenced by an array of metrics -- despite having parameter counts of the same order as Leviathan. Indeed, Fig. \ref{fig:leviathan}\textbf{d--e} illustrate the DeepONet and FNO having accumulated immense enstrophy log-ratio errors, total spectral energy log-ratio errors and spectral slope errors by $t=20$. The FNO, having decayed to a zero state, ‘‘falls off'' the spectral plots entirely. The U-Net preserves a small amount of spatial structure for longer, but nevertheless succumbs to the same drift-to-mean. Fig. \ref{fig:leviathan}\textbf{f} confirms this systemic failure, with the vorticity probability density functions of each of the three deterministic models rapidly collapsing to delta distributions, as well as in Fig. \ref{fig:rollout} with generated fields quickly flattening. Whilst the dense Transformer avoids this catastrophic collapse -- even accurately predicting enstrophy -- it remains fundamentally handicapped by its isotropic embedding approach. As is visualised in Fig. \ref{fig:rollout}, its generated fields degrade into unstructured noisy artefacts rather than exhibiting true turbulent behaviour. Without the physics-aware inductive bias of the separable primitive, these results suggest that sequence modelling alone does not suffice for a foundation model of turbulence.

By contrast, Leviathan (whose distribution is sampled via expected value) avoids these pathologies by design. As evidenced in Fig. \ref{fig:leviathan}\textbf{d}, Leviathan matches the ground truth enstrophy as well as the dense transformer, but is superior by all other metrics: conservation of spectral energy, spectral dissipation and Jensen-Shannon divergence of the vorticity probability density functions. Generated radial energy spectra closely follow those of the ground truth (Fig. \ref{fig:leviathan}\textbf{e}) and Fig. \ref{fig:leviathan}\textbf{f} shows that rollouts track the vorticity distributions of the ground truth field. Generated turbulent fields (Fig. \ref{fig:rollout}) qualitatively match the behaviour of the ground truth, sustaining distinct and coherently evolving vortex structures throughout the 20-step rollout. These results establish that treating chaotic spatiotemporal evolution as a distributional sequence modelling task -- facilitated by separable, neighbour-preserving embeddings -- effectively eliminates the off-attractor drift characteristic of deterministic pointwise operators. By exploiting an emergent factorisability of turbulent dynamics, Leviathan effects a division of computational labour: the separable primitive ensures efficient structural fidelity; global spatiotemporal reasoning is delegated to a monolithic Transformer backbone. This synergy validates our central hypothesis: composite architectures -- not pure monoliths -- are essential for grounding predictive intelligence.

\section*{Discussion}
The collective empirical performance of KHRONOS, SPAN, Janus, and Leviathan substantiates the central thesis of this work that separability is a latent property of intelligent systems, often emerging in the coordinates or representations through which the system is expressed. Whether as a standalone approximator, a variational trial space or a structured filter within a composite architecture, separable neural architectures enable the recovery of the underlying  manifold of a target system. By factorising high-dimensional states without a concomitant loss in expressivity, this formalism reconciles the continuity of physical law with the discrete nature of neural frameworks, providing a mathematical substrate for foundational models of physics.

The SNA is deployed as part of an MLP--SNA hybrid architecture, spline-based adaptive networks (SPAN). A dense layer learns how to best disentangle raw input streams to a low-rank latent space. This bears a conceptual resemblance to Koopman operator theory \cite{lusch2018koopman, brunton2016sindy}, wherein nonlinear dynamics are expressed in coordinates that admit simpler evolution operators. SPAN is integrated into actor and critic networks within a deep deterministic policy gradient (DDPG) \cite{lillicrap2019continuous, perezgil2022drlcarla} and soft actor-critic (SAC) frameworks for autonomous control. The SNA's inductive bias enforces smooth actor and critic mappings, whilst the factorised structure yields a better-conditioned action-value landscape. This stabilises policy gradients under the compounding demands of closed-loop control. Across online benchmarks spanning classical control, continuous MuJoCo locomotion and autonomous waypoint navigation in the CARLA simulator (see Appendix) \cite{dosovitskiy2017carla}, SPAN achieves $30-50\%$ improvements in sample efficiency and improved success rates ranging from $1.3-9\times$ over parameter-matched MLP baselines \cite{mostakim2026agile}. On offline expert datasets, SPAN outperforms the MLP baseline by an average factor of $6.7\times$. 

In the generative inversion of multiscale materials via Janus, an ablation confirms that the SNA head is indeed the critical driver of inversion quality. It outperforms a parameter-matched MLP baseline by $42-441\%$ in FFT-validated stiffness errors across design points (see Appendix). This advantage is mechanistic: the multilinear Jacobian of the SNA produces a better-conditioned loss landscape during inversion, reducing entrapment in the scattered local minima afflicting the more entangled MLP. Nevertheless, gradient hallucination -- the recovery of latent codes that ‘‘trick'' the predictor whilst diverging from true physics -- remains an open problem for both architectures; Janus mitigates it without resolution. Promising avenues include adversarial training of the predictive head, explicit Jacobian regularisation, physics-informed latent penalties and active learning. Resolving hallucination more fundamentally would reduce ensemble sizes required for confidence. It would also extend the approach to regimes poorly covered by training data, such as the high-porosity regime, where gradient fidelity is weakest.

Whilst the mathematical formalism of the separable neural architecture encompasses a broad representational class, empirical validations were deliberately restricted to a more fundamental instantiation (Canonical Polyadic structure with univariate B-spline atoms). That this instance is sufficient to achieve state-of-the-art performance across disparate fields underscores the generalisability of this inductive bias. Exploring higher-order interaction structures, alternative basis functions and more decompositions within this formalism remains a frontier for future inquiry. 

Despite these advantages, the practical implementation of this formalism presents open challenges. Whilst separability demonstrably emerges in the studied systems, identifying separable representations or -- more pertinently -- tokenisation schemes to expose this structure remains a non-trivial challenge. In particular, as made most explicit in Fig. \ref{fig:leviathan}\textbf{a}, Leviathan is particularly effective at extracting structure from continuous tokenisation schemes. The isotropic cloud seen in Fig. \ref{fig:leviathan}\textbf{c}, however, illustrates that arbitrary token indexing suppresses the neighbourhood structure upon which the separable primitive relies -- isolating tokenisation, not architecture, as the bottleneck for language. Even so, Leviathan demonstrates strong language modelling performance \cite{batley2026leviathan}: evaluation on the Pile dataset across the $60-420$M parameter regime yields $6.7-18.1\%$ reductions in perplexity. This is equivalent to the performance of a dense model up to twice Leviathan's size. The path forward is clear: a structure-aware tokenisation scheme capable of preserving neighbourhood relations in linguistic space. Such a scheme would expose the separability that the turbulence results confirm the primitive stands ready to exploit -- and that these results suggest lies latent in language itself.

\section*{Methods}
\label{sec:methods}
\paragraph{Preliminaries.}
For the present work, the ambient dimension is denoted $d\in\mathbb{N}$ with coordinates $[0,1]^d$. An interactivity hyperparameter $k\leq d$ is also introduced, confining the maximum order of featurewise interaction. For any subset $S\subseteq[d]$ with $|S|\leq k$, $x_S$ denotes the projection of $x$ onto $S$. A hyperparameter $r$ constrains interaction rank. Let $\rho:\mathbb{R}\rightarrow\mathbb{R}$ be an activation function. 

\paragraph{Foundational constituents.} The architecture is constructed from learnable functional components, herein termed \emph{atoms}. Formally, an atom is defined as the learnable function $\phi^{(S)}:[0,1]^{|S|}\times\Theta_S\rightarrow\mathbb{R}$, parameterised by $\theta_S$. Interactions between these atoms are subsequently governed by an \emph{interaction object} $C$; this is a collection of coefficients $c_S$ assigned to specific subsets of coordinates $S \subseteq [d]$. This object stores the abstract set of permissible featurewise interactions within the model. In particular, $C$ admits a canonical embedding into a $k$-sparse, order-$d$ interactivity tensor via the mapping $\mathcal{E}:\{c_S\}\rightarrow\mathcal{T}\in\mathbb{R}^{d\times\dots\times d}$. The \emph{rank} of this interactivity tensor $\mathcal{T}$ is strictly bounded by $r$, and it is nonzero only for subsets $|S|\leq k$.

\paragraph{Separable neural architecture.}
A Separable Neural Architecture (SNA), parameterised by $\Theta=\{\theta_S,c_S\}$, is a mapping $f:[0,1]^d\times\Theta\rightarrow\mathbb{R}$ defined as 
\begin{align}
    f(x;c_S;\theta_S)=\rho\left(\sum_{S\in\operatorname{Supp}(C)}c_S\phi^{(S)}(x_S;\theta_S)\right).
\end{align}
This function represents an element of the functional class
\begin{equation}
\begin{aligned}
    \mathcal{F}_{k,r}=\{f(x,\,&\Theta):\operatorname{rank}(\mathcal{E}(C))\leq r, \\
    &|S|\leq k, \forall S\in\operatorname{Supp}(C)\}.
\end{aligned}
\end{equation}
\paragraph{Recovery of model families.}
This definition generalises several contemporary model families based on the interactivity order $k$:
\begin{itemize}
    \item \textbf{Generalised additive models.} Recovered when $k=1$;
    \item \textbf{Generalised quadratic models.} Recovered when $k=2$;
    \item \textbf{Canonical-Polyadic-decomposed models.} Recovered when $k=d,~|S|=d$, atoms are products of univariate sub-atoms and $\operatorname{rank}(\mathcal{E}(C))\leq r$.
\end{itemize}

\paragraph{Tensorised classes.}

A subclass of SNAs, the tensor decomposition (TD) class is defined by setting $k=d$, thus $S=[d]$, and embedding the interaction object as a tensor decomposition with $C=\mathcal{E}^{-1}(\mathcal{T})$ and rank $r$. An element of this class takes the form 
\begin{align}
    f(x;\Phi_{TD})=\rho\left(\sum_{j=1}^rc^{(j)}\phi^{(j)}(x; \theta^{(j)})\right).
\end{align}
This class specialises further into its most fundamental subclass, the Canonical Polyadic (CP)-class. Here atoms are restricted to be products of univariate sub-atoms $\psi^{(j)}_i$. An element of this class is written
\begin{align}
    f(x;\Phi_{CP})=\rho\left(\sum_{j=1}^rc^{(j)}\prod_{i=1}^d\psi_i^{(j)}(x_i; \theta_i^{(j)})\right).
\end{align}
Crucially, this structural restriction does not forfeit expressivity. Consider the functional class of CP-SNAs $\mathcal{F}^r_{CP}$ under identity activation $\rho(x)=x$. Then, if each sub-atom $\psi^{(j)}_i$ is continuous on the unit segment, the union over finite ranks $\mathcal{F}=\bigcup_{r=1}^\infty\mathcal{F}^r_{CP}$ is uniformly convergent. That is to say $\mathcal{F}$ is dense in $C\left([0,1]^d\right)$ with respect to the infinity norm $\|\cdot\|_\infty$. $L_p$-convergence, that $\mathcal{F}$ is dense in $L_p\left([0,1]^d\right)$ for any $p\geq1$, immediately follows.

As the CP decomposition is the most fundamental tensor decomposition, it follows as corollary that TD-class SNAs are also dense in $C\left([0,1]^d\right)$. This includes the Tucker \cite{tucker1966} and Tensor train \cite{oseledets2011} decompositions, amongst others \cite{kolda2009}. Notably, then, every target function can be approximated arbitrarily well by some TD-class SNA with finite rank $r$.

\paragraph{Variational domains and trial spaces.}

In extending this formalism to variational problems, let $\Omega \subset \mathbb{R}^n$ be an $n$-dimensional spatial domain, $[0, T]$ a temporal interval and $\mathcal{P} \subset \mathbb{R}^m$ a parametric space. The combined domain is the Cartesian product $\mathcal{X} = \Omega \times [0, T] \times \mathcal{P}$, with ambient dimension $d = n + 1 + m$. Coordinates are indexed $i \in \{1, \dots, d\}$ covering space, time and parameters respectively. The separable trial space is the Hilbert tensor product $V=\bigotimes_{i=1}^dV^{(i)}$. For a semi-linear form $a:V\times V\rightarrow\mathbb{R}$ and linear functional $\ell:V\rightarrow\mathbb{R}$, $u\in V$ is sought such that $a(u,v)=\ell(v)$ for all $v\in V$.

\paragraph{Variational separable neural architecture.}

A variational separable neural architecture (VSNA), parameterised by $\Theta$, is a trial function $u \in V$ defined as
\begin{equation}
    u(x; c_S; \theta_S) = \sum_{S\in\operatorname{Supp}(C)} c_S \phi^{(S)}(x_S;\theta_S)
\end{equation}
where each atom $\phi^{(S)} \in V^{(S)} = \bigotimes_{i \in S} V^{(i)}$ respects the local variational structure of its coordinates. This trial function represents an element of the finite-dimensional approximation subspace $\mathcal{F}_{k,r} = \left\{ u(x, \Theta) \in V : \operatorname{rank}(\mathcal{E}(C)) \leq r,\ |S| \leq k,\forall S\in\operatorname{Supp}(C)\right\}$.

\paragraph{Variational tensorised classes.}

In a spatiotemporal--parametric domain $\mathcal{X}$ with ambient dimension $d$, let the variational trial space be written $V = \bigotimes_{i=1}^d V^{(i)}$, with each coordinate $x_i$ associated with the univariate functional space $V^{(i)}$. Then the class of CP-VSNA trial functions is defined by:
\begin{equation}
\begin{aligned}
    \mathcal{F}_{r} = \biggl\{ u(x; \theta) = &\sum_{j=1}^r \prod_{i=1}^d \psi_i^{(j)}(x_i; \theta) \\
    &\bigg| \;\psi_i^{(j)} \in V^{(i)},\, \theta \in \Theta \biggr\}
\end{aligned}
\end{equation}
given a learnable parameter set $\Theta$. Thus, $\mathcal{F}_{r} \subset V$ serves as the finite-dimensional approximation subspace for a Galerkin method utilizing SNA-ansatz functions.

\paragraph{Variational guarantees.}

To establish the classical validity of this trial space, let $a: V \times V \rightarrow \mathbb{R}$ be a bounded and coercive bilinear form with coercivity constant $c_0 > 0$ and boundedness constant $c_1 > 0$. Denoting a linear functional $\ell \in V^*$, and fixing the basis sub-atoms $\psi_i^{(j)}$, the VSNA formalism satisfies four core variational guarantees:
\begin{itemize}
    \item \textbf{Well-posedness:} The Galerkin approximation, $a(u_{r}, v_{r}) = \ell(v_{r})$ for all $v_{r} \in \mathcal{F}_{r}$, admits a unique solution $u_{r} \in \mathcal{F}_{r}$.
    \item \textbf{Quasi-optimality:} Let $u \in V$ be the unique solution to the exact weak problem. The VSNA Galerkin solution $u_{r}$ is quasi-optimal, bounded strictly by the best approximation within the trial space: $\|u - u_{r}\|_V \leq \frac{c_1}{c_0} \min_{v_{r} \in \mathcal{F}_{r}} \|u - v_{r}\|_V$.
    \item \textbf{Convergence:} If each univariate sub-atom family $\psi^{(j)}_i(\cdot~; \theta)$ is dense in $V^{(i)}$, then $\bigcup_r\mathcal{F}_{r}$ is dense in $V$ with respect to the Hilbert norm $\|\cdot\|_V$. Consequently, as the interaction rank $r \to \infty$, the approximation error $\varepsilon_r \to 0$, ensuring VSNA Galerkin solutions converge to the exact solution $u \in V$.
    \item \textbf{Stability:} The VSNA Galerkin solution $u_{r}$ satisfies the absolute stability bound $\|u_{r}\|_V \leq \frac{1}{c_0} \|\ell\|_{V^*}$, where the dual norm is defined as $\|\ell\|_{V^*} = \sup_{v\in V, v\neq 0} \frac{\ell(v)}{\|v\|_V}$.
\end{itemize}
Taken together, these guarantees prove that for any sufficiently regular spatiotemporal--parametric problem, the VSNA forms a well-posed, quasi-optimal, stable and convergent trial space.

\section*{Declarations}

\begin{itemize}
\item Funding

S. Saha gratefully acknowledges start up funding by Kevin T Crofton Department of Aerospace and Ocean Engineering.

\item Data availability 

The Inconel 718 \cite{xie2021mechanistic}, L-BOM \cite{wang2026lbom}, PDEBench \cite{takamoto2022pdebench} and Sketch-to-stress \cite{yu2024sketch2stress} datasets are all publicly available via their respective publications. 
\item Code availability 

Code will be made publicly available upon publication.
\item Author contributions

R.T.B.: Conceptualisation, Methodology, Formal analysis, Resources, Software (KHRONOS, VSNA, Janus, Leviathan), Investigation, Visualisation, Writing -- original draft, Writing -- review \& editing.

A.S.: Software (Sketch-to-stress adaptation of Janus), Investigation, Visualisation, Writing -- review \& editing.

R.M.: Software (SPAN), Investigation, Writing -- review \& editing.

A.K.: Software (Inconel 718 adaptation of KHRONOS), Investigation.

S.S.: Conceptualisation, Resources, Supervision, Writing -- review \& editing.
\end{itemize}

\noindent

\bibliographystyle{unsrtnat}
\bibliography{sn-bibliography}%

@misc{batley2025khronos,
      title={KHRONOS: a Kernel-Based Neural Architecture for Rapid, Resource-Efficient Scientific Computation}, 
      author={Reza T. Batley and Sourav Saha},
      year={2025},
      eprint={2505.13315},
      archivePrefix={arXiv},
      primaryClass={cs.LG},
      url={https://arxiv.org/abs/2505.13315}, 
}

@article{xie2021mechanistic,
  title={Mechanistic data-driven prediction of as-built mechanical properties in metal additive manufacturing},
  author={Xie, X. and Bennett, J. and Saha, S. and others},
  journal={npj Computational Materials},
  volume={7},
  number={1},
  pages={86},
  year={2021},
  publisher={Nature Publishing Group},
  doi={10.1038/s41524-021-00555-z},
  url={https://doi.org/10.1038/s41524-021-00555-z}
}

@article{fang2022data,
  title={Data-driven analysis of process, structure, and properties of additively manufactured Inconel 718 thin walls},
  author={Fang, L. and Cheng, L. and Glerum, J. A. and others},
  journal={npj Computational Materials},
  volume={8},
  number={1},
  pages={126},
  year={2022},
  publisher={Nature Publishing Group},
  doi={10.1038/s41524-022-00808-5},
  url={https://doi.org/10.1038/s41524-022-00808-5}
}

@article{park2025unifying,
  author    = {Park, Chanwook and Saha, Sourav and Guo, Jiachen and Zhang, Hantao and Xie, Xiaoyu and Bessa, Miguel A. and Qian, Dong and Chen, Wei and Wanger, Gregory J. and Cao, Jian and Hughes, Thomas J. R. and Liu, Wing Kam},
  title     = {Unifying machine learning and interpolation theory via interpolating neural networks},
  journal   = {Nature Communications},
  year      = {2025},
  volume    = {16},
  number    = {1},
  pages     = {8753},
  doi       = {10.1038/s41467-025-63790-8},
  issn      = {2041-1723}
}

@inbook{batley2026janus,
    author = {Reza T. Batley and Sourav Saha},
    title = {A Unified Generative-Predictive Framework for Deterministic Inverse Design},
    publisher = {AIAA SciTech Forum},
    year = {2026},
    chapter = {},
    pages = {},
    doi = {10.2514/6.2026-0365},
    URL = {https://arc.aiaa.org/doi/abs/10.2514/6.2026-0365},
}

@inbook{sarker2026khronos,
    author = {Apurba Sarker and Reza T. Batley and Darshan Sarojini and Sourav Saha},
    title = {A Kernel-based Resource-efficient Neural Surrogate for Multi-fidelity Prediction of Aerodynamic Field},
    publisher = {AIAA SciTech Forum},
    year = {2026},
    chapter = {},
    pages = {},
    doi = {10.2514/6.2026-0043},
    URL = {https://arc.aiaa.org/doi/abs/10.2514/6.2026-0043},
}

@misc{mostakim2026agile,
      title={Agile Reinforcement Learning through Separable Neural Architecture}, 
      author={Rajib Mostakim and Reza T. Batley and Sourav Saha},
      year={2026},
      eprint={2601.23225},
      archivePrefix={arXiv},
      primaryClass={cs.LG},
      url={https://arxiv.org/abs/2601.23225}, 
}

@misc{batley2026leviathan,
      title={A Separable Architecture for Continuous Token Representation in Language Models}, 
      author={Reza T. Batley and Sourav Saha},
      year={2026},
      eprint={2601.22040},
      archivePrefix={arXiv},
      primaryClass={cs.CL},
      url={https://arxiv.org/abs/2601.22040}, 
}

@inproceedings{takamoto2022pdebench,
    author = {Takamoto, Makoto and Praditia, Timothy and Leiteritz, Raphael and MacKinlay, Daniel and Alesiani, Francesco and Pfl\"{u}ger, Dirk and Niepert, Mathias},
    booktitle = {Advances in Neural Information Processing Systems},
    pages = {1596--1611},
    publisher = {Curran Associates, Inc.},
    title = {PDEBench: An Extensive Benchmark for Scientific Machine Learning},
    url = {https://proceedings.neurips.cc/paper_files/paper/2022/file/0a9747136d411fb83f0cf81820d44afb-Paper-Datasets_and_Benchmarks.pdf},
    volume = {35},
    year = {2022}
}

@data{takamoto2022dataset,
author = {Takamoto, Makoto and Praditia, Timothy and Leiteritz, Raphael and MacKinlay, Dan and Alesiani, Francesco and Pflüger, Dirk and Niepert, Mathias},
publisher = {DaRUS},
title = {{PDEBench Datasets}},
year = {2022},
version = {V8},
doi = {10.18419/DARUS-2986},
url = {https://doi.org/10.18419/DARUS-2986}
}

@misc{li2021fno,
      title={Fourier Neural Operator for Parametric Partial Differential Equations}, 
      author={Zongyi Li and Nikola Kovachki and Kamyar Azizzadenesheli and Burigede Liu and Kaushik Bhattacharya and Andrew Stuart and Anima Anandkumar},
      year={2021},
      eprint={2010.08895},
      archivePrefix={arXiv},
      primaryClass={cs.LG},
      url={https://arxiv.org/abs/2010.08895}, 
}

@article{lu2021deeponet,
   title={Learning nonlinear operators via DeepONet based on the universal approximation theorem of operators},
   volume={3},
   ISSN={2522-5839},
   url={http://dx.doi.org/10.1038/s42256-021-00302-5},
   DOI={10.1038/s42256-021-00302-5},
   number={3},
   journal={Nature Machine Intelligence},
   publisher={Springer Science and Business Media LLC},
   author={Lu, Lu and Jin, Pengzhan and Pang, Guofei and Zhang, Zhongqiang and Karniadakis, George Em},
   year={2021},
   month=mar, pages={218–229} }

@article{mandl2025sepdeeponet,
    title = {Separable physics-informed DeepONet: Breaking the curse of dimensionality in physics-informed machine learning},
    journal = {Computer Methods in Applied Mechanics and Engineering},
    volume = {434},
    pages = {117586},
    year = {2025},
    issn = {0045-7825},
    doi = {https://doi.org/10.1016/j.cma.2024.117586},
    url = {https://www.sciencedirect.com/science/article/pii/S0045782524008405},
    author = {Luis Mandl and Somdatta Goswami and Lena Lambers and Tim Ricken}
}

@inproceedings{jiang2023training,
    author = {Jiang, Ruoxi and Lu, Peter Y. and Orlova, Elena and Willett, Rebecca},
    title = {Training neural operators to preserve invariant measures of chaotic attractors},
    year = {2023},
    publisher = {Curran Associates Inc.},
    address = {Red Hook, NY, USA},
    booktitle = {Proceedings of the 37th International Conference on Neural Information Processing Systems},
    articleno = {1202},
    numpages = {25},
    location = {New Orleans, LA, USA},
    series = {NIPS '23}
}

@Article{scher2019weather,
    AUTHOR = {Scher, S. and Messori, G.},
    TITLE = {Weather and climate forecasting with neural networks: using general circulation models (GCMs) with different complexity as a study ground},
    JOURNAL = {Geoscientific Model Development},
    VOLUME = {12},
    YEAR = {2019},
    NUMBER = {7},
    PAGES = {2797--2809},
    URL = {https://gmd.copernicus.org/articles/12/2797/2019/},
    DOI = {10.5194/gmd-12-2797-2019}
}

@inproceedings{Cao2021transformer,
    author        = {Shuhao Cao},
    title         = {Choose a Transformer: {F}ourier or {G}alerkin},
    booktitle     = {Advances in Neural Information Processing Systems (NeurIPS 2021)},
    volume        = {34},
    year          = {2021},
    eprint        = {arXiv: 2105.14995},
    primaryclass  = {cs.CL},
    url={https://openreview.net/forum?id=ssohLcmn4-r},
}

@misc{wang2022prefixlm,
      title={What Language Model Architecture and Pretraining Objective Work Best for Zero-Shot Generalization?}, 
      author={Thomas Wang and Adam Roberts and Daniel Hesslow and Teven Le Scao and Hyung Won Chung and Iz Beltagy and Julien Launay and Colin Raffel},
      year={2022},
      eprint={2204.05832},
      archivePrefix={arXiv},
      primaryClass={cs.CL},
      url={https://arxiv.org/abs/2204.05832}, 
}

@article{raissi2019pinn,
    title = {Physics-informed neural networks: A deep learning framework for solving forward and inverse problems involving nonlinear partial differential equations},
    journal = {Journal of Computational Physics},
    volume = {378},
    pages = {686--707},
    year = {2019},
    issn = {0021-9991},
    doi = {https://doi.org/10.1016/j.jcp.2018.10.045},
    url = {https://www.sciencedirect.com/science/article/pii/S0021999118307125},
    author = {M. Raissi and P. Perdikaris and G.E. Karniadakis},
}

@article{Chinesta2011,
  author  = {Chinesta, Francisco and Ladevèze, Pierre and Cueto, Elías},
  title   = {A Short Review on Model Order Reduction Based on Proper Generalized Decomposition},
  journal = {Archives of Computational Methods in Engineering},
  year    = {2011},
  volume  = {18},
  number  = {4},
  pages   = {395--404},
  doi     = {10.1007/s11831-011-9064-7},
  url     = {https://doi.org/10.1007/s11831-011-9064-7}
}

@inproceedings{chen2016xgboost, series={KDD ’16},
   title={XGBoost: A Scalable Tree Boosting System},
   url={http://dx.doi.org/10.1145/2939672.2939785},
   DOI={10.1145/2939672.2939785},
   booktitle={Proceedings of the 22nd ACM SIGKDD International Conference on Knowledge Discovery and Data Mining},
   publisher={ACM},
   author={Chen, Tianqi and Guestrin, Carlos},
   year={2016},
   month=aug, pages={785–794},
   collection={KDD ’16} }

@article{bendsoe1999material,
  author  = {Bends{\o}e, Martin P. and Sigmund, Ole},
  title   = {Material Interpolation Schemes in Topology Optimization},
  journal = {Archive of Applied Mechanics},
  year    = {1999},
  volume  = {69},
  number  = {9-10},
  pages   = {635--654},
  doi     = {10.1007/s004190050248}
}

@article{fleck2010micro,
  author  = {Fleck, Norman A. and Deshpande, Vikram S. and Ashby, Michael F.},
  title   = {Micro-architectured Materials: Past, Present and Future},
  journal = {Proceedings of the Royal Society A: Mathematical, Physical and Engineering Sciences},
  year    = {2010},
  volume  = {466},
  number  = {2121},
  pages   = {2495--2516},
  doi     = {10.1098/rspa.2010.0215}
}

@article{bendsoe1988generating,
  author  = {Bends{\o}e, Martin P. and Kikuchi, Noboru},
  title   = {Generating Optimal Topologies in Structural Design Using a Homogenization Method},
  journal = {Computer Methods in Applied Mechanics and Engineering},
  year    = {1988},
  volume  = {71},
  number  = {2},
  pages   = {197--224},
  doi     = {10.1016/0045-7825(88)90086-2}
}

@article{rodrigues2002hierarchical,
  author  = {Rodrigues, H. and Fernandes, P. and Guedes, J. M.},
  title   = {A Hierarchical Optimization of Material and Structure},
  journal = {Structural and Multidisciplinary Optimization},
  year    = {2002},
  volume  = {24},
  number  = {1},
  pages   = {1--10},
  doi     = {10.1007/s00158-002-0209-z}
}

@article{feyel2000fe2,
  author  = {Feyel, F. and Chaboche, J.-L.},
  title   = {FE2 Multiscale Approach for Modelling the Elastoviscoplastic Behaviour of Long Fibre SiC/Ti Composite Materials},
  journal = {Computer Methods in Applied Mechanics and Engineering},
  year    = {2000},
  volume  = {183},
  number  = {3-4},
  pages   = {309--330},
  doi     = {10.1016/S0045-7825(99)00224-8}
}

@article{panetta2015elastic,
  author  = {Panetta, Julian and Zhou, Qingnan and Malomo, Luigi and Pietroni, Nico and Cignoni, Paolo and Zorin, Denis},
  title   = {Elastic Textures for Additive Fabrication},
  journal = {ACM Transactions on Graphics},
  year    = {2015},
  volume  = {34},
  number  = {4},
  pages   = {135:1--135:12},
  doi     = {10.1145/2766937}
}

@article{wang2026lbom,
  author  = {Wang, Lili and Feng, Jingxuan and Zhai, Xiaoya and Han, Jiacheng and Chen, Kai and Ma, Winston Wai Shing and Liu, Ligang and Fu, Xiao-Ming},
  title   = {Data-driven inverse design of multifunctional bicontinuous multiscale structures},
  journal = {Nature Communications},
  year    = {2026},
  volume  = {17},
  number  = {1},
  doi     = {10.1038/s41467-025-68089-2}
}

@article{wang2020deep,
  author  = {Wang, Wenjie and Yu, Xiaoyu and Zheng, Xian and others},
  title   = {Deep Generative Modeling for Mechanistic-Based Learning and Design of Metamaterial Systems},
  journal = {Computer Methods in Applied Mechanics and Engineering},
  year    = {2020},
  volume  = {372},
  pages   = {113377},
  doi     = {10.1016/j.cma.2020.113377}
}

@article{bastek2022inverting,
  author  = {Bastek, Jan H. and Kochmann, Dennis M. and others},
  title   = {Inverting the Structure--Property Map of Truss Metamaterials by Deep Learning},
  journal = {Extreme Mechanics Letters},
  year    = {2022},
  volume  = {53},
  pages   = {101700},
  doi     = {10.1016/j.eml.2022.101700}
}

@article{garner2019compatibility,
  author  = {Garner, Evan and Kolken, H. M. A. and Wang, Chen and Zadpoor, Amir A. and Wu, Jun},
  title   = {Compatibility in Microstructural Optimization for Additive Manufacturing},
  journal = {Additive Manufacturing},
  year    = {2019},
  volume  = {28},
  pages   = {425--434},
  doi     = {10.1016/j.addma.2019.05.021}
}

@article{zheng2021datadriven,
  author  = {Zheng, Xiaoyu and Yu, Xiaoyu and Wang, Wenjie and others},
  title   = {Data-driven Multiscale Design of Cellular Materials with Tailored Mechanical Properties},
  journal = {Computer Methods in Applied Mechanics and Engineering},
  year    = {2021},
  volume  = {380},
  pages   = {113782},
  doi     = {10.1016/j.cma.2021.113782}
}

@inproceedings{bora2017compressed,
  author    = {Bora, Ashish and Jalal, Ajil and Price, Eric and Dimakis, Alexandros G.},
  title     = {Compressed Sensing Using Generative Models},
  booktitle = {Proceedings of the 34th International Conference on Machine Learning},
  year      = {2017},
  pages     = {537--546}
}

@inproceedings{yeh2017semantic,
  author    = {Yeh, Raymond A. and Chen, Chen and Lim, Teck Yian and Hasegawa-Johnson, Mark and Do, Minh N. and Pfister, Hanspeter},
  title     = {Semantic Image Inpainting with Deep Generative Models},
  booktitle = {Proceedings of the IEEE Conference on Computer Vision and Pattern Recognition},
  year      = {2017},
  pages     = {5485--5493},
  doi       = {10.1109/CVPR.2017.579}
}

@misc{dosovitskiy2017carla,
      title={CARLA: An Open Urban Driving Simulator}, 
      author={Alexey Dosovitskiy and German Ros and Felipe Codevilla and Antonio Lopez and Vladlen Koltun},
      year={2017},
      eprint={1711.03938},
      archivePrefix={arXiv},
      primaryClass={cs.LG},
      url={https://arxiv.org/abs/1711.03938}, 
}

@misc{lillicrap2019continuous,
      title={Continuous control with deep reinforcement learning}, 
      author={Timothy P. Lillicrap and Jonathan J. Hunt and Alexander Pritzel and Nicolas Heess and Tom Erez and Yuval Tassa and David Silver and Daan Wierstra},
      year={2019},
      eprint={1509.02971},
      archivePrefix={arXiv},
      primaryClass={cs.LG},
      url={https://arxiv.org/abs/1509.02971}, 
}

@article{perezgil2022drlcarla,
  author  = {P{\'e}rez-Gil, {\'O}scar and Barea, Rafael and L{\'o}pez-Guill{\'e}n, Elena and Bergasa, Luis M. and G{\'o}mez-Hu{\'e}lamo, Carlos and Guti{\'e}rrez, Rodrigo and D{\'i}az-D{\'i}az, Alejandro},
  title   = {Deep reinforcement learning based control for autonomous vehicles in CARLA},
  journal = {Multimedia Tools and Applications},
  year    = {2022},
  volume  = {81},
  number  = {3},
  pages   = {3553--3576},
  doi     = {10.1007/s11042-021-11437-3}
}

@article{lusch2018koopman,
  author  = {Lusch, Bethany and Kutz, J. Nathan and Brunton, Steven L.},
  title   = {Deep learning for universal linear embeddings of nonlinear dynamics},
  journal = {Nature Communications},
  year    = {2018},
  volume  = {9},
  number  = {1},
  pages   = {4950},
  doi     = {10.1038/s41467-018-07210-0}
}

@article{brunton2016sindy,
  author  = {Brunton, Steven L. and Proctor, Joshua L. and Kutz, J. Nathan},
  title   = {Discovering governing equations from data by sparse identification of nonlinear dynamical systems},
  journal = {Proceedings of the National Academy of Sciences},
  year    = {2016},
  volume  = {113},
  number  = {15},
  pages   = {3932--3937},
  doi     = {10.1073/pnas.1517384113}
}

@misc{liu2025kan,
      title={KAN: Kolmogorov-Arnold Networks}, 
      author={Ziming Liu and Yixuan Wang and Sachin Vaidya and Fabian Ruehle and James Halverson and Marin Soljačić and Thomas Y. Hou and Max Tegmark},
      year={2025},
      eprint={2404.19756},
      archivePrefix={arXiv},
      primaryClass={cs.LG},
      url={https://arxiv.org/abs/2404.19756}, 
}

@article{saha2021hidenn,
  title   = {Hierarchical deep learning neural network (HiDeNN): An artificial intelligence (AI) framework for computational science and engineering},
  author  = {Saha, Sourav and Gan, Zhengtao and Cheng, Lin and Gao, Jiaying and Kafka, Orion L. and Xie, Xiaoyu and Li, Hengyang and Tajdari, Mahsa and Kim, H. Alicia and Liu, Wing Kam},
  journal = {Computer Methods in Applied Mechanics and Engineering},
  volume  = {373},
  pages   = {113452},
  year    = {2021},
  doi     = {10.1016/j.cma.2020.113452}
}

@article{zhang2022hidenntd,
    title = {HiDeNN-TD: Reduced-order hierarchical deep learning neural networks},
    journal = {Computer Methods in Applied Mechanics and Engineering},
    volume = {389},
    pages = {114414},
    year = {2022},
    issn = {0045-7825},
    doi = {https://doi.org/10.1016/j.cma.2021.114414},
    url = {https://www.sciencedirect.com/science/article/pii/S0045782521006629},
    author = {Lei Zhang and Ye Lu and Shaoqiang Tang and Wing Kam Liu}
}

@article{yu2024sketch2stress,
    author = {Deng Yu and Chufeng Xiao and Manfred Lau and Hongbo Fu},
    title = {Sketch2Stress: Sketching with Structural Stress Awareness},
    journal = {IEEE Transactions on Visualization and Computer Graphics},
    year = {2024},
    volume = {30},
    number = {10},
    pages = {6851--6865},
    doi = {10.1109/TVCG.2023.3342119},
    URL = {https://doi.org/10.1109/TVCG.2023.3342119},
}

@article{tucker1966,
  title={Some mathematical notes on three-mode factor analysis},
  author={Tucker, Ledyard R},
  journal={Psychometrika},
  volume={31},
  number={3},
  pages={279--311},
  year={1966},
  publisher={Springer}
}

@article{kolda2009,
  title={Tensor decompositions and applications},
  author={Kolda, Tamara G and Bader, Brett W},
  journal={SIAM Review},
  volume={51},
  number={3},
  pages={455--500},
  year={2009},
  publisher={SIAM}
}

@article{oseledets2011,
  title={Tensor-train decomposition},
  author={Oseledets, Ivan V},
  journal={SIAM Journal on Scientific Computing},
  volume={33},
  number={5},
  pages={2295--2317},
  year={2011},
  publisher={SIAM}
}

@misc{goodfellow2015explaining,
      title={Explaining and Harnessing Adversarial Examples}, 
      author={Ian J. Goodfellow and Jonathon Shlens and Christian Szegedy},
      year={2015},
      eprint={1412.6572},
      archivePrefix={arXiv},
      primaryClass={stat.ML},
      url={https://arxiv.org/abs/1412.6572}, 
}

@misc{raffel2023exploring,
      title={Exploring the Limits of Transfer Learning with a Unified Text-to-Text Transformer}, 
      author={Colin Raffel and Noam Shazeer and Adam Roberts and Katherine Lee and Sharan Narang and Michael Matena and Yanqi Zhou and Wei Li and Peter J. Liu},
      year={2023},
      eprint={1910.10683},
      archivePrefix={arXiv},
      primaryClass={cs.LG},
      url={https://arxiv.org/abs/1910.10683}, 
}

@book{piegl1997nurbs,
    author = {Piegl, Les and Tiller, Wayne},
    title = {The NURBS Book},
    publisher = {Springer},
    year = {1997},
    edition = {2},
    series = {Monographs in Visual Communication},
    doi = {10.1007/978-3-642-59223-2}
}

@book{vallis2017atmospheric,
    title={Atmospheric and oceanic fluid dynamics: fundamentals and large-scale circulation},
    author={Vallis, Geoffrey K},
    year={2017},
    publisher={Cambridge University Press}
}

@article{breiman2001random,
  title={Random forests},
  author={Breiman, Leo},
  journal={Machine learning},
  volume={45},
  pages={5--32},
  year={2001},
  publisher={Springer}
}

\appendix

\onecolumn

\section{Universality of tensor decomposed separable neural architectures}

\begin{theorem}[Universal approximation of CP-separable neural architectures]
\label{thm:cpuniv}
Let $\Omega=[0,1]^d$. For each coordinate $i\in\{1,\dots,d\}$, let $U^{(i)}\subset C([0,1])$
be a unital subalgebra that is dense in $C([0,1])$ with respect to $\|\cdot\|_\infty$.
Define
\begin{align}
    \mathcal A = \left\{\sum_{j=1}^r \prod_{i=1}^d a^{(i)}_j(x_i): a^{(i)}_j\in U^{(i)}, r\in\mathbb N \right\}.
\end{align}
Then $\mathcal A$ is dense in $C([0,1]^d)$ with respect to the uniform norm.
\end{theorem}

\begin{proof}
For each coordinate $i\in\{1,\dots,d\}$, let $\tilde U^{(i)}=\operatorname{span}\{\phi^{(i)}_\alpha : \alpha=1,\dots,r\}$ be a function space dense in $C([0,1])$, and let $U^{(i)}=\operatorname{alg}(\tilde U^{(i)})$ be the univariate algebra generated by $\tilde U^{(i)}$, with $1\in U^{(i)}$. Pick arbitrary $f,g\in\mathcal{A}$ such that
\begin{align}
    f(x)=\sum_{j=1}^{r}\ \prod_{i=1}^{d} \alpha^{(i)}_{j}(x_i),~~ g(x)=\sum_{k=1}^{s}\ \prod_{i=1}^{d} \beta^{(i)}_{k}(x_i),
\end{align}
with $\alpha^{(i)}_{j},\beta^{(i)}_{k}\in U^{(i)}$. Then,
\begin{align}
    f(x)+g(x)&=\sum_{j=1}^{r}\ \prod_{i=1}^{d} \alpha^{(i)}_{j}(x_i)+\sum_{k=1}^{s}\ \prod_{i=1}^{d} \beta^{(i)}_{k}(x_i),\\
    &=\sum_{l=1}^{r+s}\prod_{i=1}^d\gamma_l^{(i)}(x_i)\quad\text{where }~\gamma_l^{(i)}\in U^{(i)}.
\end{align}
Therefore $f+g\in\mathcal{A}$, confirming closure under addition. Take an arbitrary element $f\in\mathcal{A}$ as before, and multiply by an arbitrary scalar $\lambda\in\mathbb{R}$ so
\begin{align}
    \lambda f(x)=\sum_{j=1}^{r}\lambda \prod_{i=1}^{d} \alpha^{(i)}_{j}(x_i)=\sum_{j=1}^{r} \prod_{i=1}^{d} \tilde\alpha^{(i)}_{j}(x_i),
\end{align}
with $\tilde\alpha^{(1)}_{j}=\lambda\alpha^{(1)}_{j}$, all else unchanged. Since $U^{(i)}$ is a vector space, $\lambda\alpha^{(1)}_{j}\in U^{(i)}$, thus $\lambda f\in\mathcal{A}$ and $\mathcal{A}$ is closed under scalar multiplication. Pick $f,g\in\mathcal{A}$ as before. Then, their product is
\begin{align}
    (fg)(x)&=\left(\sum_{j=1}^{r}\ \prod_{i=1}^{d} \alpha^{(i)}_{j}(x_i)\right)\left(\sum_{k=1}^{s}\ \prod_{i=1}^{d} \beta^{(i)}_{k}(x_i)\right),\\
    &=\sum_{j=1}^{r}\sum_{k=1}^{s}\prod_{i=1}^{d}\gamma_{jk}^{(i)}(x_i),~ \gamma_{jk}^{(i)}=\alpha^{(i)}_{j}(x_i)\beta^{(i)}_{k}(x_i).
\end{align}
Recall $U^{(i)}=\operatorname{alg}(\tilde U^{(i)})$ is a unital algebra and so it follows by definition that $\gamma_{jk}^{(i)}\in U^{(i)}$. Therefore, $(fg)(x)\in\mathcal{A}$, and $\mathcal{A}$ is an algebra: closed under multiplication. Since $1\in U^{(i)}$, it follows that $\mathcal{A}$ is a unital algebra.

Take two distinct points $x,y\in[0,1]^d$, $x=(x_1,\dots,x_d),~y=(y_1,\dots,y_d), $ with $x_i\neq y_i$ for at least one $i=1,\dots,d$. Now, because $U^{(i)}$ is dense in $C([0,1])$ and $C([0,1])$ separates points, $U^{(i)}$ also separates points: $\exists~u\in U^{(i)}$ with $u(x_i)\neq u(y_i)$. Then the separable function $f(z)=u(z_i)\in\mathcal{A}$ and satisfies $f(x)\neq f(y)$. Therefore, $\mathcal{A}$ separates points.

Any unital subalgebra on $C([0,1]^d)$ that separates points is itself dense in $C([0,1]^d)$, given that $[0,1]^d$ is a compact Hausdorff space. This follows from Stone-Weierstrass. Indeed, $\mathcal{A}$ is dense in $C([0,1]^d)$; for a given $f^*\in C([0,1]^d)$ and any $\varepsilon>0$, density yields $f\in\mathcal{A}$ such that
\begin{align}
    \|f^*-f\|_\infty<\varepsilon.
\end{align}
By definition of $\mathcal{A}$, $f$ is a \emph{finite} sum of products. Therefore, there exists $r$ and functions $\{\phi^{(i)}_{j}\}$ such that
\begin{align}
\sup_{x\in[0,1]^d}\left|f^*(x)-\sum_{j=1}^{r}\prod_{i=1}^{d}\phi^{(i)}_{j}(x_i)\right|<\varepsilon,
\end{align}
which is the desired universal approximation statement.
\end{proof}

\begin{corollary}[Universal approximation of tensor decomposed separable neural architectures]
\label{cor:tduniv}
Let $\mathcal{F}_{TD}$ represent the functional class of tensor decomposed (TD)-separable neural architectures on $[0,1]^d$. Then $\mathcal{F}_{TD}$ is dense in $C([0,1]^d)$ with respect to the uniform norm $\|\cdot\|_\infty$.
\end{corollary}

\begin{proof}
By definition, the Canonical Polyadic (CP) class $\mathcal{A}$ is a specific, restricted instantiation of the broader TD-class $\mathcal{F}_{TD}$. The general multidimensional atoms $\phi^{(j)}(x)$ permitted in the TD-class are strictly constrained in the CP-class to be factorised products of univariate subatoms, $\prod_{i=1}^d\phi^{(i)}_{j}(x_i)$. 

Because any valid CP-class function is inherently a valid TD-class function, the functional classes satisfy the strict inclusion relation $\mathcal{A} \subseteq \mathcal{F}_{TD}$. 

Theorem \ref{thm:cpuniv} establishes that the subset $\mathcal{A}$ is dense in $C([0,1]^d)$ under the uniform norm. It is a fundamental property of dense sets that any superset containing a dense subset is itself dense in the same space. Therefore, the superset $\mathcal{F}_{TD}$ must also be dense in $C([0,1]^d)$. 

Consequently, for any target function $f^* \in C([0,1]^d)$ and any $\varepsilon > 0$, there exists an approximator $f \in \mathcal{F}_{TD}$ with a finite rank $r$ such that $\|f^* - f\|_\infty < \varepsilon$.
\end{proof}

\section{Theoretical guarantees of the variational separable neural architecture}

Let $V$ be a Hilbert space corresponding to the spatiotemporal--parametric domain $\mathcal{X}$, equipped with the norm $\|\cdot\|_V$. Consider a variational problem defined by a bilinear form $a: V \times V \to \mathbb{R}$ and a linear functional $\ell: V \to \mathbb{R}$. The exact weak formulation seeks $u \in V$ such that
\begin{equation}
    a(u,v) = \ell(v) \quad \forall v \in V.
\end{equation}

The standard assumptions of regularity are imposed on the governing operator:
\begin{enumerate}
    \item \textbf{Boundedness:} There exists a constant $c_1 > 0$ such that $|a(u,v)| \leq c_1 \|u\|_V \|v\|_V$ for all $u,v \in V$.
    \item \textbf{Coercivity:} There exists a constant $c_0 > 0$ such that $a(v,v) \geq c_0 \|v\|_V^2$ for all $v \in V$.
    \item \textbf{Bounded linear functional:} $\ell \in V^*$, meaning there exists a constant $M > 0$ such that $|\ell(v)| \leq M \|v\|_V$ for all $v \in V$.
\end{enumerate}

Under these conditions, the Lax-Milgram theorem guarantees a unique exact solution $u \in V$. The finite-dimensional Canonical Polyadic (CP) trial space $\mathcal{F}_r \subset V$ of rank $r$ is constructed from univariate subatoms $\psi_i^{(j)} \in V^{(i)}$, and the Galerkin approximation $u_r \in \mathcal{F}_r$ is sought satisfying:
\begin{equation}
    a(u_r, v_r) = \ell(v_r) \quad \forall v_r \in \mathcal{F}_r.
\end{equation}

\begin{theorem}[Well-posedness]
    The VSNA Galerkin approximation admits a unique solution $u_r \in \mathcal{F}_r$.
\end{theorem}
\begin{proof}
    For a fixed rank $r$ and fixed basis representations of the subatoms, the VSNA trial space $\mathcal{F}_r$ constitutes a finite-dimensional subspace of the Hilbert space $V$. Because $\mathcal{F}_r \subset V$, the bilinear form $a(\cdot,\cdot)$ remains bounded and coercive when restricted to $\mathcal{F}_r \times \mathcal{F}_r$. Similarly, the linear functional $\ell$ remains bounded on $\mathcal{F}_r$. By the application of the Lax-Milgram theorem to the finite-dimensional closed subspace $\mathcal{F}_r$, there exists a unique $u_r \in \mathcal{F}_r$ that satisfies the restricted variational problem.
\end{proof}

\begin{theorem}[Quasi-optimality]
\label{thm:qo}
    Let $u \in V$ be the exact solution and $u_r \in \mathcal{F}_r$ be the VSNA Galerkin solution. The approximation error is strictly bounded by the best possible approximation within the separable trial space:
    \begin{equation}
        \|u - u_r\|_V \leq \frac{c_1}{c_0} \min_{v_r \in \mathcal{F}_r} \|u - v_r\|_V.
    \end{equation}
\end{theorem}
\begin{proof}
    Because $\mathcal{F}_r \subset V$, $v = v_r \in \mathcal{F}_r$ may be chosen as a test function in the exact problem, yielding $a(u, v_r) = \ell(v_r)$. Subtracting the Galerkin formulation $a(u_r, v_r) = \ell(v_r)$ from this yields the fundamental Galerkin orthogonality condition:
    \begin{equation}
        a(u - u_r, v_r) = 0 \quad \forall v_r \in \mathcal{F}_r.
    \end{equation}
    For any function $v_r \in \mathcal{F}_r$, coercivity gives:
    \begin{equation}
        c_0 \|u - u_r\|_V^2 \leq a(u - u_r, u - u_r).
    \end{equation}
    Adding and subtracting $v_r$, and applying the bilinearity of $a(\cdot,\cdot)$ gives
    \begin{equation}
        a(u - u_r, u - u_r) = a(u - u_r, u - v_r) + a(u - u_r, v_r - u_r).
    \end{equation}
    Since $(v_r - u_r) \in \mathcal{F}_r$, the second term vanishes due to Galerkin orthogonality. Applying the boundedness of the bilinear form to the remaining term yields
    \begin{equation}
        c_0 \|u - u_r\|_V^2 \leq a(u - u_r, u - v_r) \leq c_1 \|u - u_r\|_V \|u - v_r\|_V.
    \end{equation}
    Dividing by $c_0 \|u - u_r\|_V$ gives $\|u - u_r\|_V \leq \frac{c_1}{c_0} \|u - v_r\|_V$. Because this holds for any $v_r \in \mathcal{F}_r$, the minimum over all $v_r \in \mathcal{F}_r$ may be taken, recovering Céa's Lemma for the separable neural architecture.
\end{proof}

\begin{theorem}[Convergence]
    Assume $k=d$. If each univariate subatom family $\psi_i^{(j)}$ is dense in its respective coordinate space $V^{(i)}$, the VSNA Galerkin solution $u_r$ converges to the exact solution $u$ as the interaction rank $r \to \infty$. That is, $\lim_{r \to \infty} \|u - u_r\|_V = 0$.
\end{theorem}
\begin{proof}
    As established in Corollary \ref{cor:tduniv} and its extension to Hilbert tensor product $V=\bigotimes_{i=1}^dV^{(i)}$, the canonical polyadic tensor class $\bigcup_{r=1}^\infty \mathcal{F}_r$ is dense in the joint Hilbert space $V = \bigotimes_{i=1}^d V^{(i)}$ under the assumption that the univariate bases are dense in their respective 1D spaces. Therefore, for any $\varepsilon > 0$, there exists a finite rank $r^*$ and an approximator $v_{r^*} \in \mathcal{F}_{r^*}$ such that the best-approximation error satisfies:
    \begin{equation}
        \min_{v_{r^*} \in \mathcal{F}_{r^*}} \|u - v_{r^*}\|_V < \frac{c_0}{c_1} \varepsilon.
    \end{equation}
    Applying the quasi-optimality bound from Theorem \ref{thm:qo}, the error of the Galerkin solution $u_{r^*}$ is bounded by:
    \begin{equation}
        \|u - u_{r^*}\|_V \leq \frac{c_1}{c_0} \left( \frac{c_0}{c_1} \varepsilon \right) = \varepsilon.
    \end{equation}
    As rank $r \to \infty$, the best-approximation error monotonically approaches zero, guaranteeing that the Galerkin solution converges to the exact weak solution $u$.
\end{proof}

\begin{theorem}[Stability]
    The VSNA Galerkin solution $u_r$ satisfies the absolute stability bound:
    \begin{equation}
        \|u_r\|_V \leq \frac{1}{c_0} \|\ell\|_{V^*}.
    \end{equation}
\end{theorem}
\begin{proof}
    Choosing $v_r = u_r$ in the Galerkin formulation yields $a(u_r, u_r) = \ell(u_r)$. Applying the coercivity of $a(\cdot,\cdot)$ to the left-hand side and the definition of the dual norm $\|\ell\|_{V^*} = \sup_{v \neq 0} \frac{\ell(v)}{\|v\|_V}$ to the right-hand side gives
    \begin{equation}
        c_0 \|u_r\|_V^2 \leq a(u_r, u_r) = \ell(u_r) \leq \|\ell\|_{V^*} \|u_r\|_V.
    \end{equation}
    Dividing both sides by $c_0 \|u_r\|_V$ (assuming $u_r \neq 0$) yields the desired stability bound, demonstrating that the learned separable representation is continuously dependent on the problem data and cannot diverge unless the source functional itself is unbounded.
\end{proof}

\section{Tensor-native optimisation}

Direct solution of resulting global Galerkin systems $Au=b$ over $d$-dimensional domains requires assembling an operator of size $O(N^d \times N^d)$, where $N$ is the number of per-dimension degrees of freedom. For the six-dimensional spatiotemporal--parametric problems considered, this global assembly is computationally intractable. To bypass this -- the ‘‘curse of dimensionality'' -- a tensor-native alternating least-squares (ALS) approach is applied. By enforcing the Canonical Polyadic (CP) structure on the trial space, the globally non-linear Galerkin projection is solved via a sequence of linear, one-dimensional updates. 

Fixing all dimensions other than $d$ and solving for $\theta^{(d)}$ alone, the local ALS system is written
\begin{equation}
    \left( {A}^{(d)} + \lambda I \right) \mathrm{vec}(\theta^{(d)}) = b^{(d)}
\end{equation}
where $\lambda$ is a small Tikhonov regularisation parameter introduced to ensure strict positive-definiteness of the local operator. 

Crucially, the local stiffness matrix $A^{(d)}$ and load vector $b^{(d)}$ are assembled without ever constructing the global operators. Instead, they are computed exactly via Kronecker products of one-dimensional integral matrices from the active dimension and the contracted interaction weights from the fixed dimensions. 

Physical constraints are imposed explicitly. Spatial Dirichlet boundary conditions and temporal initial conditions are strongly enforced on the relevant univariate subatoms prior to each local solve. This is possible by lifting the initial condition $u(\cdot,t=0)=u_0$ to write $u=u'+u_0$ and solving for $u'$ with the temporal subatom initially respecting a homogeneous Dirichlet condition. To prevent numerical underflow or overflow across the high-order tensor products, the active subatoms are re-normalised after each update, with the scalar magnitudes symmetrically absorbed into the remaining fixed dimensions.

This iterative procedure sweeps through all $d$ dimensions sequentially until the relative change in the global residual norm falls below a specified tolerance. By operating entirely within the factorised latent space, the tensor-native formulation reduces the computational complexity of solving the high-dimensional PDE from $\mathcal{O}(N^{3d})$ (for a direct global solve) to $\mathcal{O}(I_{max} \cdot d \cdot R^2 N)$, where $I_{max}$ is the number of ALS iterations. This effectively eliminates the exponential scaling bottleneck of traditional grid-based discretisations.

\section{Standalone benchmarks}

The following benchmarks accompany examples in the main text to further validate the CP-class SNA/VSNA as a standalone architecture. Two regimes are considered: high-dimensional supervised regression, in which KHRONOS is evaluated against contemporary surrogates on problems where ground truth is known exactly; and variational PDE solution, in which a CP-class VSNA is extended beyond the coercive bilinear setting of the established theory to a nonlinear, shock-forming operator. All experiments are run on commodity CPU hardware.

\subsection{Supervised regression}

KHRONOS is evaluated against Random Forest (RF) \cite{breiman2001random}, XGBoost \cite{chen2016xgboost} and a multilayer perceptron (MLP) on two benchmark problems of increasing difficulty. In each case, model complexity was increased until validation $R^2=0.999$ to provide a consistent saturation point for comparison.

\paragraph{8-dimensional borehole function.} The borehole function,
\begin{align}
    \label{eq:borehole}
    u(p)=2\pi p_3(p_4-p_6)\left(\log\left(\frac{p_2}{p_1}\right) \left(1+2\frac{p_7p_3}{\log\left(\frac{p_2}{p_1}\right)p_1^2p_8}+\frac{p_3}{p_5}\right) \right)^{-1},
\end{align}
is a standard surrogate modelling benchmark. Data are generated at 100,000 points via Latin Hypercube sampling, normalised and split 70/30 train--test. Table \ref{tab:borehole} summarises results. RF saturates at $R^2=0.9969$, failing to reach the target. KHRONOS achieves $R^2=0.9998$ with only 240 trainable parameters -- over four orders of magnitude fewer than RF, $350\times$ fewer than XGBoost and $20\times$ fewer than an MLP. KHRONOS is the only model to reach the target accuracy in under a second.

\begin{table*}[!ht]
\centering
\caption{Regression benchmark on the 8D borehole problem.}
\label{tab:borehole}
\begin{tabular}{lcccc}
\toprule
Metric & Random Forest\cite{breiman2001random} & XGBoost \cite{chen2016xgboost} & MLP & KHRONOS \\
\midrule
Trainable parameters & 4,261,376 & 84,600 & 5601 & 240\\
Training time, s       & 2.8 & 1.5 & 22 & 0.87\\
Test MSE                & $1.0\times 10^{-4}$ & $3.3\times 10^{-5}$ & $2.8\times 10^{-5}$ & $2.2\times 10^{-5}$ \\
Test R$^2$              & 0.9969 & 0.9990 & 0.9992 & 0.9998\\
\bottomrule
\end{tabular}
\end{table*}

\paragraph{20-dimensional noisy Sobol-G function}

The Sobol-G function,
\begin{align}
    u(p)&=\prod_{i=1}^{20}\frac{|4p_i -2|+a_i}{1+a_i},\\
    a_i&=
    \begin{cases}
        0&\quad\textrm{for }i=1,\dots,5\\
        \frac32&\quad\textrm{for }i=6,\dots,10\\
        4&\quad\textrm{for }i=11,\dots,20
    \end{cases},\\
    p&=[0, 1]^{20},
\end{align}
is itself fully separable: a product of univariate terms, and therefore lies within only a finite-dimensional representational space instantiated by KHRONOS. This makes it a direct empirical test of the central claim of this work -- that should it exist, the SNA is the primitive to exploit latent separable structure. Outputs are corrupted additive noise $\epsilon\sim\mathcal{N}(0,0.01^2)$ and evaluated at 100,000 LHS points. Results are given in Table \ref{tab:sobol}. RF, XGBoost and the MLP all saturate far below the target accuracy; this is not for want of parameters, but misaligned inductive biases. KHRONOS reaches $R^2=0.9994$ with 1,560 parameters, trained in 5.1 seconds, directly exploiting structure the other models cannot see.

\begin{table*}[!ht]
\centering
\caption{Regression benchmark on a noisy 20D Sobol-G Function.}
\label{tab:sobol}
\begin{tabular}{lcccc}
\toprule
Metric & Random Forest \cite{breiman2001random} & XGBoost \cite{chen2016xgboost} & MLP & KHRONOS \\
\midrule
Trainable parameters & 8,849,208 & 239,564 & 13,569 & 1560\\
Training time (s)       & 5.4 & 5.6 & 66 & 5.1\\
Test MSE& $4.6\times 10^{-4}$ & $3.3\times 10^{-5}$ & $1.4\times 10^{-4}$ & $6.8\times 10^{-7}$ \\
Test $R^2$              & 0.5565 & 0.7312 & 0.8788 & 0.9994\\
\bottomrule
\end{tabular}
\end{table*}

\subsection{Variational PDE solution}

This section presents a qualitatively harder regime than what is established in the main text, in the form of a one-dimensional inviscid Burgers' equation,
\begin{align}
    \partial_tu+u\partial_x u=0,\quad(x,t)\in[0,1]^2.
\end{align}
The nonlinear advection term $u\partial_x u$ produces a trilinear form in the weak formulation, placing this problem outside not only the coercive but the bilinear setting of the established theory.

It is solved by a CP-class VSNA with a Fourier basis in the spatial coordinate -- strongly enforcing Dirichlet boundary conditions - and orthonormalised Chebyshev basis in time. The initial condition is strongly enforced via a lifting of the trial space to satisfy $\hat u(x,0)=u_0(x)$. Whilst this is a spectral combination of basis functions, the standard exponential rate of convergence is not expected here; the operator is nonlinear. The weak residual,
\begin{align}
\label{eq:weak}
    \mathcal{R}\{\hat{u}, v\} = \int_0^1\int_0^1 v(
    \partial_t \hat{u} + \hat{u}\,\partial_x \hat{u} )~dx~dt = 0,
    \quad \forall v \in V,
\end{align}
is minimised over a finite test space via L-BFGS. The separable structure of the trial space reduces \eqref{eq:weak} to contractions over one-dimensional quadratures to avoid grid-based assembly. It is notable that this reduction is agnostic to the order of the resulting multilinear form: whether the weak formulation is bilinear, as in linear PDE theory, or trilinear, as here, the trial space accordingly decomposes each integral into the same class of one-dimensional atom contractions. The inherent cost is in the proliferation of cross terms growing as $O(r^n)$ in the trial space rank $r$ and form order $n$. 

The result is shown in Fig. \ref{fig:burgers} -- the VSNA accurately recovers the smooth pre-shock solution, with absolute errors below 0.02 across the domain. The spatially ringing residual structure visible as $x,t\rightarrow1$ is consistent with characteristic convergence at the domain corner at an impending shock; Gibbs phenomenon.

\begin{figure}
    \centering
    \includegraphics[width=\linewidth]{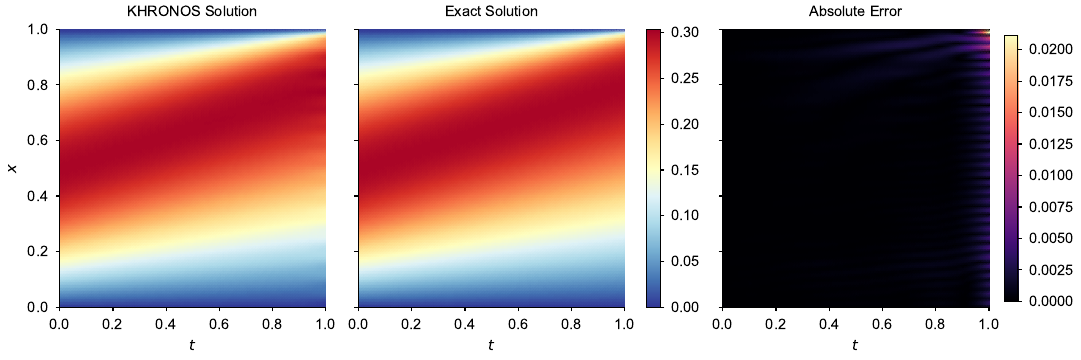}
    \caption{\textbf{Variational solution of the inviscid Burgers' equation by a CP-class VSNA.} The recovered solution (left) and exact solution (centre) show close agreement across the full spatiotemporal domain. Absolute errors remain below 0.02 throughout; the spatially ringing residual structure visible as $x,t\rightarrow1$ is consistent with impending shock collapse.}
    \label{fig:burgers}
\end{figure}

\section{Reinforcement learning for autonomous control}

Dynamic, closed-loop sequential decision-making represents a challenge for neural function approximation, as minor predictive inaccuracies cascade over time. An autonomous waypoint-navigation task is formulated within the CARLA \cite{dosovitskiy2017carla} driving simulator. Navigation is controlled by a Deep Deterministic Policy Gradient (DDPG) framework \cite{lillicrap2019continuous, perezgil2022drlcarla}. The effect of network topology is isolated by replacing both deterministic actor ($\mu(s|\theta^\mu)$) and action-value critic ($Q(s,a|\theta^Q)$) networks with a spline-based adaptive network (SPAN) \cite{mostakim2026agile}. SPAN is a composite adaptation of KHRONOS, in which a dense layer learns how to best disentangle raw input streams to a low-rank latent space. This bears a conceptual resemblance to Koopman operator theory \cite{lusch2018koopman, brunton2016sindy}, wherein nonlinear dynamics are expressed in coordinates that admit simpler evolution operators. Specifically, nonlinear dynamics are mapped to a latent observation space in which evolution is well approximated by a separable representation, viz. KHRONOS.

In this problem, the architectures were tasked with learning mappings from environmental observations to vehicular controls. At each timestep, the networks processed a normalised state vector comprising the vehicle's longitudinal velocity, its relative heading angle and the lateral coordinates of 15 upcoming trajectory waypoints. From this state, the deterministic actors output a continuous two-dimensional control signal dictating steering and throttle actuation. To succeed, the agents must continuously minimise lateral deviation from the route whilst avoiding failure conditions: collisions, lane invasions and timeouts. The MLP baseline has 1,917 total parameters divided between its actor and critic networks, with SPAN containing 1,901.

Empirical results from the benchmark highlight the sensitivity of dense architectures in closed-loop continuous control, shown in Fig. \ref{fig:span}\textbf{a}. Across 60 evaluation episodes spanning 20 distinct routes, the parameter-matched MLP baseline achieves a global goal completion rate of only $50\%$. This fragility likely arises from the interaction between policy and value approximation in DDPG. When the critic is represented by an unconstrained, monolithic network, the resulting action-value landscape becomes irregular. This produces unstable policy gradients for the actor. In closed-loop driving, these instabilities compound over time. As shown in Fig. \ref{fig:span}\textbf{b,c}, the MLP agent rapidly diverges from the target waypoints, resulting in lane invasions and thus early termination. 

SPAN's B-spline basis enforces smooth mappings in both actor and critic networks. Furthermore, by restricting featurewise interactions to factorised combinations, their input--output Jacobians are structured. This ensures the critic learns a better-behaved action-value landscape to yield more stable policy gradients for the actor. Consequently, SPAN achieves an $83.3\%$ success rate. Whilst improving generalisation, the highly stochastic nature of control does not guarantee superior performance across routes. For instance, on Route 12 (Fig. \ref{fig:span}\textbf{d}), the SPAN agent suffers a repeated early lane invasion whereas the MLP agent navigates to the endpoint. Nevertheless, the aggregate result demonstrates that the imposition of a separable inductive bias stabilises the policy-value learning dynamics of DDPG. This suggests that the separable neural primitive can enhance predictive and generative intelligence not only in canonical models, but also when embedded within a composite learning system.

\begin{figure}
    \centering
    \includegraphics[width=\linewidth]{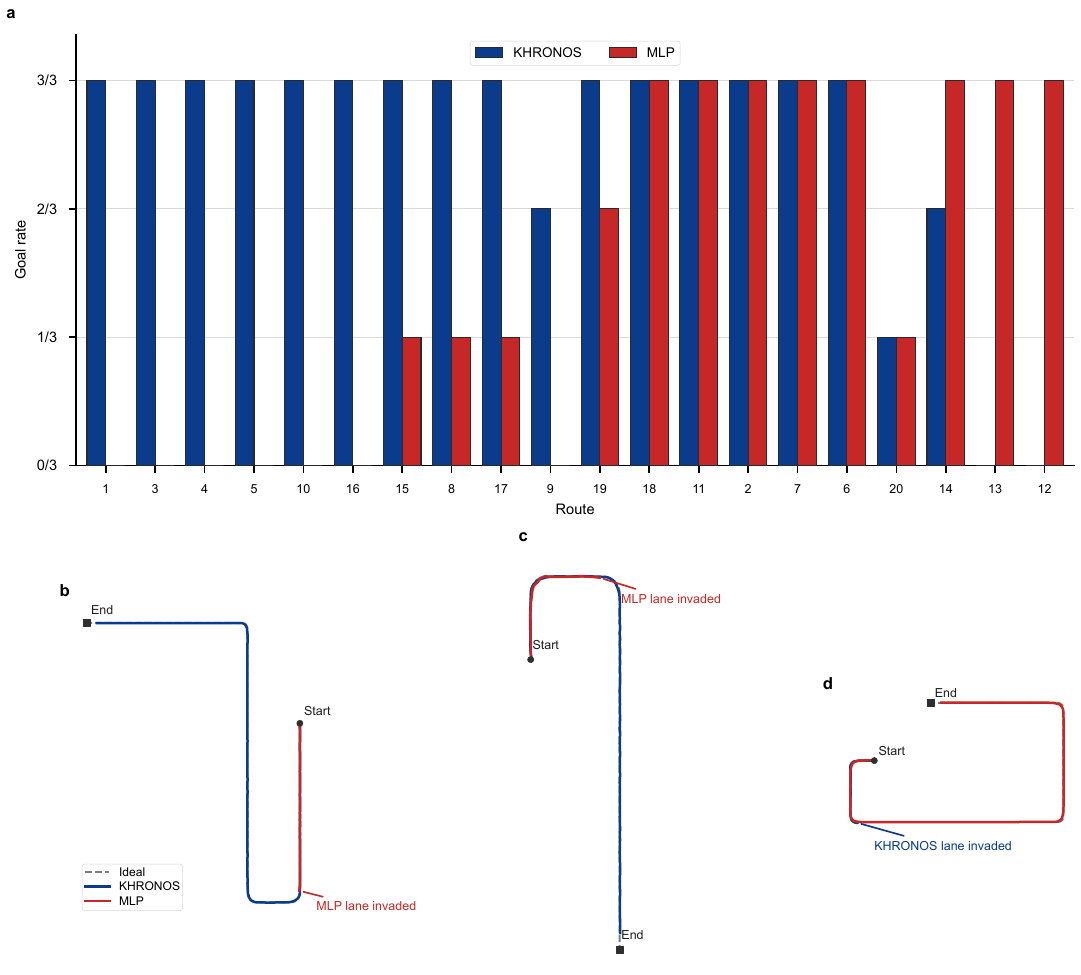}
    \caption{\textbf{Spline-based adaptive networks (SPAN) applied to closed-loop autonomous navigation.} \textbf{a}, Quantitative comparison of goal completion rates on 3 independent trials on each of the 20 routes in the CARLA simulator. SPAN achieves an 83.3\% success rate compared to 50\% for the MLP. \textbf{b, c}, Representative top-down trajectories illustrating key SPAN successes, (\textbf{b}, Route 3; \textbf{c}, Route 9). SPAN (blue) tracks the ideal trajectory (dashed grey) whilst the MLP (red) invades lanes early on, terminating the run. \textbf{d}, A highlighted failure mode (Route 12) in which SPAN experiences early lane invasions whilst the MLP navigates to the end.}
    \label{fig:span}
\end{figure}

\section{Generative inversion}

The specific contribution of Janus's separable predictive head is isolated. To demonstrate that this mechanism is beneficial, an ablation under approximate parameter parity is carried out. In particular, given the 64-dimensional latent approximated by a rank-32, 8-element CP-SNA, precisely 17,207 parameters were used. The ablation multilayer perceptron (MLP) was thus constructed with two hidden ReLU layers of widths 128 and 64, totalling 18,071 parameters. 

On forward-mode pretraining, the ablation baseline achieved an $R^2$ of 0.992, marginally above Janus's $R^2$ of 0.989. It is notable that both models incur identical walltimes for both training and inversion, confirming that any performance difference is architectural rather than computational. In evaluating invertibility, both models were tasked with generating microstructures across three representative design points from the macroscopic beam gradient: the root (high density, $V_f=0.65$, $C_{1111}=350,000$ MPa), the midpoint (medium density, $V_f=0.45$, $C_{1111}=200,000$ MPa) and the tip (high porosity, $V_f=0.25$, $C_{1111}=50,000$ MPa). Both models minimised the target loss via maximum a posteriori (MAP) estimation over 800 steps, at 16 different initial seeds.

Subsequent physical validation by means of fast Fourier transform (FFT) homogenisation revealed the magnitude of disparity between model-predicted and true physical viability, as shown in Fig. \ref{fig:ablation}\textbf{a}. The champion of each swarm -- the lowest-error generated cell -- is then analysed. Whilst both Janus and ablation baselines grossly overpredict their own performance -- a well-known phenomenon: gradient hallucination -- Janus's generation was superior at all points, ranging from $29.48\%$ better at the midpoint to $81.53\%$ better at the root. Notable is that the MLP predicts a ‘‘high'' relative error of almost 5\% at the tip. This suggests difficulty in the MLP's inversion -- the model has not successfully minimised even its own predictive error.

This difficulty is systemic across the generated distribution. The ensemble statistics, visualised as a box-and-whiskers plot in Fig. \ref{fig:ablation}\textbf{b}, reveal the broader consequences of the gradient entanglement of the MLP. With its dense, less convex Jacobian, the MLP's latent trajectories may be consistently trapped in scattered local minima. This worsens the statistics across the board; indeed, for each design point, the relative error means and quartiles of the ablation baseline's generations lie above those of Janus. The multilinear Jacobian of the CP-SNA likely mitigates such gradient trapping, consistently anchoring its ensemble closer to the true physical manifold.

\begin{figure}[!ht]
    \centering
    \includegraphics[width=\linewidth]{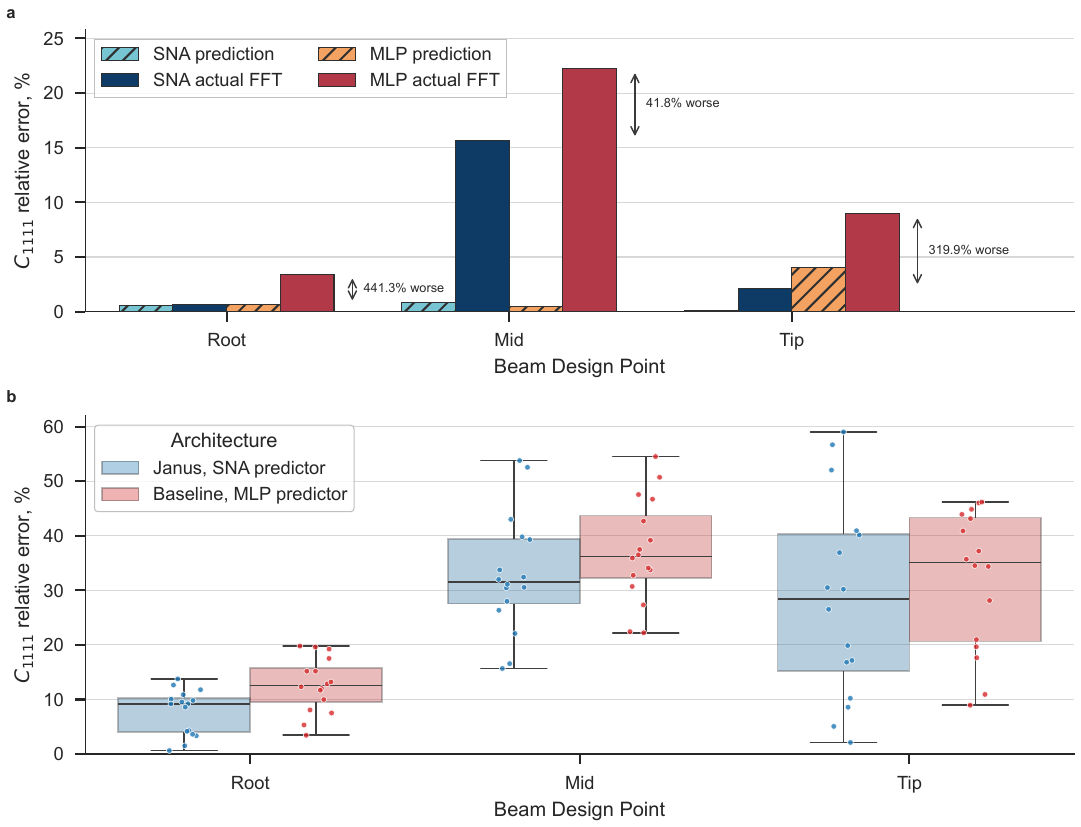}
    \caption{\textbf{Generative inversion ablation.} \textbf{a}, Comparison of predicted -- by the predictive heads -- versus actual physical errors. These are as evaluated via fast Fourier transform (FFT) homogenisation for the champion generated unit cells at three representative macroscopic design points, the root ($V_f=0.65$), the middle ($V_f=0.45$) and the tip ($V_f=0.25$). Whilst the unconstrained MLP baseline head predicts low deviations at root and midpoint (hatched orange bars), the FFT solver exposes catastrophic divergence (solid red bars). The MLP faces a more difficult inversion objective, evidenced by the stubbornly large predicted error, almost 5\%, at the tip. Janus also hallucinates, but produces significantly higher fidelity unit cells. \textbf{b}, Ensemble statistics of physical errors for the 16 randomly initialised latent seeds per design point, as visualised by a box-and-whiskers plot with overlaid swarm. Because the MLP's dense and entangled weights create a less convex Jacobian than Janus's, a higher optimisation floor associated with scattered local minima is established.}
    \label{fig:ablation}
\end{figure}

\subsection{Janus applied to Sketch-to-stress field prediction and latent inversion}
\begin{figure}[h!]
    \centering
    \includegraphics[width=\linewidth]{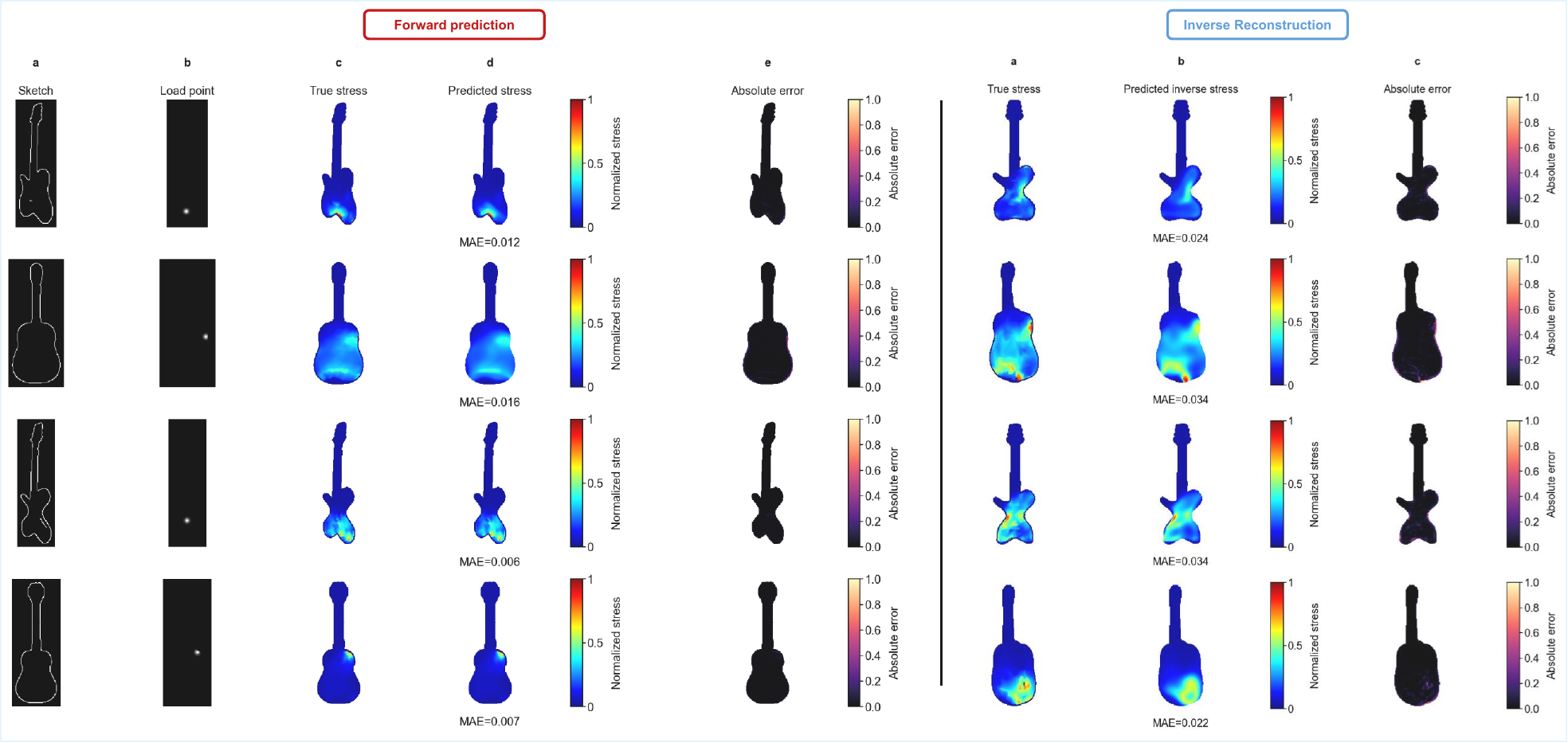}
    \caption{Sketch-to-stress forward prediction and response code reconstruction for four representative guitar configurations.Forward prediction (left): a), Input sketch. b) Load application-point map. c) Ground truth normalised stress field. d) Forward Janus prediction. e) Absolute pointwise error. Inverse reconstruction (right): a), Ground truth normalised stress field. b), Stress field reconstructed from the compact 8-dimensional response code $\mathbf{z} \in \mathbb{R}^8$. c), Absolute pointwise error. Forward predictions achieve masked MAEs of $0.006$--$0.016$, faithfully recovering dominant stress concentrations, hotspot locations, and global field topology across all four configurations. Inverse reconstructions from the response code achieve MAEs of $0.034$--$0.106$, preserving principal hotspot location and global stress topology.}
    \label{fig:janus_sketch_stress}
\end{figure}

As a supplementary demonstration, Janus is applied to the sketch-to-stress problem of Yu et al.~\cite{yu2024sketch2stress}. The publicly available guitar subset of the dataset comprises 78 unique structural geometries, each paired with approximately 118 distinct force-application locations, yielding 9,186 sketch--force--stress triplets in total. Each sample consists of a geometric sketch, a load application-point map, and a ground-truth stress field. Stress fields are recovered in scalar form by inverting the jet colormap of the dataset's RGB stress images via nearest-neighbour lookup against a precomputed colormap lookup table; background pixels are assigned zero. After masking, the resulting field expresses relative stress magnitude within the object boundary on a normalised $[0,1]$ scale, with $0$ denoting negligible response and $1$ corresponding to peak stress. Because a single base geometry underlies many force locations, the learning problem requires the model to infer not only the object's structural topology but also how the spatial stress distribution reorganises as the point of loading migrates across the structure.

The Janus framework is instantiated as follows. An encoder processes the two-channel input, consisting of the sketch and Gaussian application-point map, through four successive downsampling blocks into a shared bottleneck representation $\mathbf{x}_b$. A decoder then maps this bottleneck representation back to the full stress field for the forward prediction task. In parallel, two predictor branches operate on the pooled bottleneck features: a scalar peak-stress head that isolates the maximum response magnitude, and an 8-dimensional latent head that compresses the stress distribution into a compact response code $\mathbf{z} \in \mathbb{R}^8$.

The forward model is trained to reconstruct the stress field from the encoded geometric and loading information, while the latent head is trained jointly through a self-consistency penalty that anchors the predicted latent code $\mathbf{z}$ to the pooled bottleneck representation produced by the encoder. This encourages the latent branch to preserve the dominant structural information of the stress field while remaining compact and inversion-friendly.

As shown in Fig.~\ref{fig:janus_sketch_stress}, the forward branch achieves masked MAEs of $0.006$--$0.016$, faithfully recovering dominant stress concentrations, hotspot locations, and global field topology across structurally distinct configurations and qualitatively different loading patterns. The inverse branch, operating exclusively from the 8-dimensional code, achieves MAEs of $0.034$--$0.106$. The principal hotspot, global stress topology, and peak magnitude are preserved; visible degradation is confined to lower-response regions where absolute errors remain small. The roughly $2$--$3\times$ increase in MAE from forward to inverse is consistent with the information loss incurred by compressing the full spatial bottleneck into eight scalar coordinates, and confirms that $\mathbf{z}$ retains the mechanically salient features of the field whilst discarding fine spatial detail. This application illustrates a natural instantiation of the shared-latent principle in computational mechanics: geometry and loading condition are specified through a purely graphical interface, whilst a single latent pathway simultaneously supports dense field prediction and compact response-code-based reconstruction.

\section{Distributional sequence modelling of turbulence}

Leviathan is trained with the objective of maximising the log-likelihood of the ground-truth vorticity field under the predicted distribution. Formally, let a quantised vorticity field be tokenised into a sequence $u(t)=\{x_1,\dots,x_N\}$ where each token $x_{(\cdot)}$ is a spatial degree of freedom. The conditional distribution is thus parameterised by $p_\theta(x_{t+1,i} | u(t), x_{t+1,<i}),$ with $\theta$ denoting Leviathan's parameters. These are optimised by minimising the standard autoregressive cross-entropy loss objective,
\begin{align}
    \mathcal{L}(\theta)=-\mathbb{E}\!\left[\sum_{i}\log p_\theta\!\left(x_{t+1,i} \mid u(t), x_{t+1,<i}\right)\right].
\end{align}
In this work, the sequence length is fixed at $N=4096$, corresponding to a flattened representation of a $64\times64$ spatial grid. To leverage the high-resolution $512\times512$ turbulence data, the domain is partitioned into $8\times 8$ non-overlapping subdomains. These 64 distinct patches are treated as independent spatial streams to greatly amplify the training corpus. In this setting, then, Leviathan is tasked with learning local, translation-invariant rules rather than global flow structures. 

A direct physical implication is that Leviathan operates on a free-boundary corpus where the original simulations were periodic. This increases the problem's inherent entropy: the global system is conservative, but the local patches function as open thermodynamic systems where quantities typically conserved -- mass, vorticity, energy -- flux freely across boundaries.

Turbulent rollout proceeds via autoregressive generation. The next predicted spatial token $x_{t+1,i}$ is sampled from the output distribution and appended to the context window for subsequent predictions. This exposes Leviathan to the Lyapunov instability of the chaotic system, wherein infinitesimal errors compound over successive autoregressive steps. Computationally, this simulation -- unavoidably serial -- makes use of Key-Value caching to reduce the complexity of each autoregressive step from $O(N^2)$ to $O(N)$.

\end{document}